\documentclass[11pt]{article}

\usepackage[preprint]{acl}

\usepackage{times}
\usepackage{latexsym}

\usepackage[T1]{fontenc}
\usepackage[utf8]{inputenc}

\usepackage{microtype}

\usepackage{inconsolata}

\usepackage{graphicx}
\usepackage{subfigure}
\usepackage{amsmath}
\usepackage{amssymb}
\usepackage{bm}
\usepackage{amsthm}
\newtheorem{proposition}{Proposition}
\usepackage{booktabs}
\usepackage{multirow}
\usepackage{color, colortbl}
\definecolor{Gray}{gray}{0.9}
\usepackage{subcaption}
\usepackage[ruled,vlined,linesnumbered]{algorithm2e}

\title{GeoMin: Data-Efficient Semi-Supervised RLVR \\ 
via Geometric Distribution Modeling}

\author{
  \textbf{Guangcheng Zhu\textsuperscript{1,2}}\quad
  \textbf{Shenzhi Yang\textsuperscript{1,2}}\quad
  \textbf{Haobo Wang\textsuperscript{1}\thanks{Corresponding author.}}\quad
  \textbf{Xing Zheng\textsuperscript{2}}\\
  \textbf{Yingfan MA\textsuperscript{2}}\quad
  \textbf{Xuening Feng\textsuperscript{2}}\quad
  \textbf{Zhongqi Chen\textsuperscript{2}}\quad
  \textbf{Kai Tang\textsuperscript{12}}\\
  \textbf{Zhengqing Zang\textsuperscript{12}}\quad
  \textbf{Bowen Song\textsuperscript{2}}\footnotemark[1]\quad
  \textbf{Weiqiang Wang\textsuperscript{2}}\quad
  \textbf{Gang Chen\textsuperscript{1}}
 \\
  \textsuperscript{1}Zhejiang University \quad \textsuperscript{2}Ant Group\\
}

\begin{document}
\maketitle
{\let\thefootnote\relax\footnotetext{Code is available at \href{https://github.com/gczhu/GeoMin}{https://github.com/gczhu/GeoMin}.}}
\begin{abstract}
Reinforcement learning with verifiable rewards (RLVR) significantly advances LLM reasoning, yet it faces a dilemma: standard supervised scaling is throttled by high annotation costs, while unsupervised alternatives suffer from severe model collapse.
Recent semi-supervised RLVR methods address this by using a small labeled set to guide unlabeled data, achieving a promising trade-off between training efficacy and annotation cost.
However, they suffer from a severe data-efficiency bottleneck due to the reliance on coarse performance heuristics, leaving a vast majority of valuable instances underutilized.
To this end, we propose \textbf{GeoMin}, which models global feature distributions on labeled data to decode the structural discrepancy between correct and incorrect rollouts, thereby establishing a robust prior to assess the reliability of self-reward signals and fully unleash the potential of unlabeled data.
Empirically, GeoMin outperforms the strongest baselines by \textbf{+4.1\%} and even surpasses fully supervised models with only \textbf{10\%} of the annotations, demonstrating remarkable data efficiency.
\end{abstract}

\section{Introduction}
Reinforcement learning with verifiable rewards (RLVR) has emerged as a pivotal paradigm in advancing the reasoning capabilities of large language models (LLMs)~\citep{jaech2024openai, guo2025deepseek, comanici2025gemini, yang2025qwen3}.
By grounding rewards in verifiable outcomes, RLVR prioritizes reasoning paths that lead to correct answers, effectively eliciting trustworthy reasoning~\citep{wen2025reinforcement,zhang2025survey}.
However, this efficacy relies on the manual annotation of ground-truth answers, a process notoriously time-consuming and labor-intensive. 
The challenge is further exacerbated in specialized domains such as medicine and finance, where expert-dependent labeling is prohibitively expensive.
Therefore, the annotation bottleneck severely limits the scalability of RLVR in the era of rapidly evolving LLMs~\citep{su2025crossing,zhao2026absolute,he2026far}.

This has spurred recent interest in unsupervised RLVR, which derives rewards from the model’s internal confidence signals like majority voting~\citep{zuo2025ttrl}, entropy~\citep{agarwal2025unreasonable}, and self-certainty~\citep{zhao2025learning}, thereby entirely eliminating the reliance on ground-truth annotations.
While capable of delivering early gains, several studies~\citep{shafayat2025can,zhang2025no,zhang2025co,he2026far} show that these unsupervised methods invariably suffer from reward hacking and model collapse. 
This failure stems from a naive reliance on unfiltered data, where intrinsically noisy self-reward signals reinforce false confidence divorced from factual accuracy. 
Thus, the model merely engages in self-deception without acquiring genuine reasoning capabilities, rendering the practical applicability of such methods infeasible.

\begin{figure*}[!t]
	\centering
    \subfigure[\small Selection efficacy comparison\label{fig:low_yield}]{
		\includegraphics[width=0.344\textwidth]{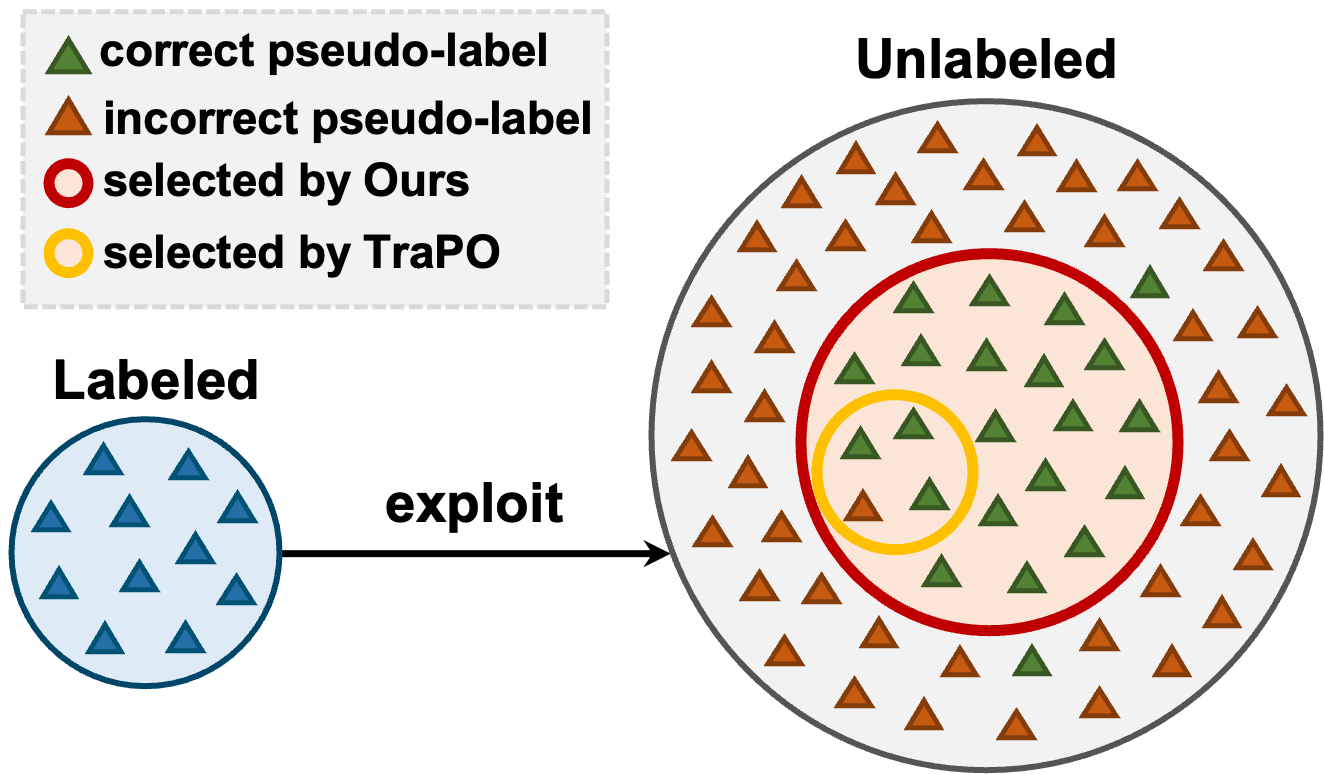}
	}
    \subfigure[\small Directional separation dynamics\label{fig:diff_curve}]{
		\includegraphics[width=0.354\textwidth]{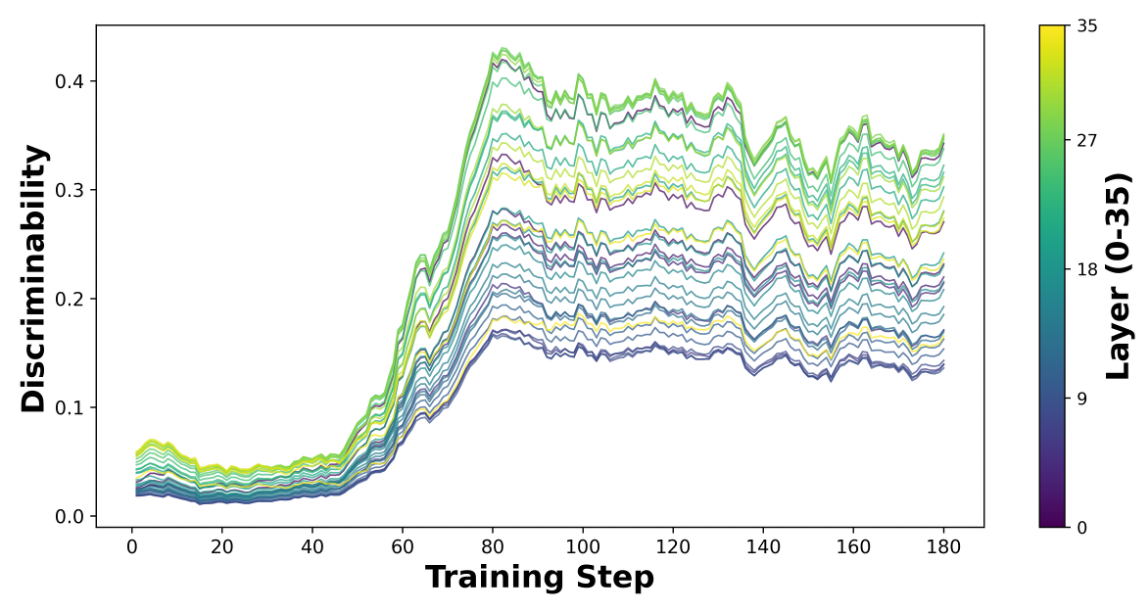}
	}
    \subfigure[\small Geometric resonance\label{fig:dist_match}]{
		\includegraphics[width=0.254\textwidth]{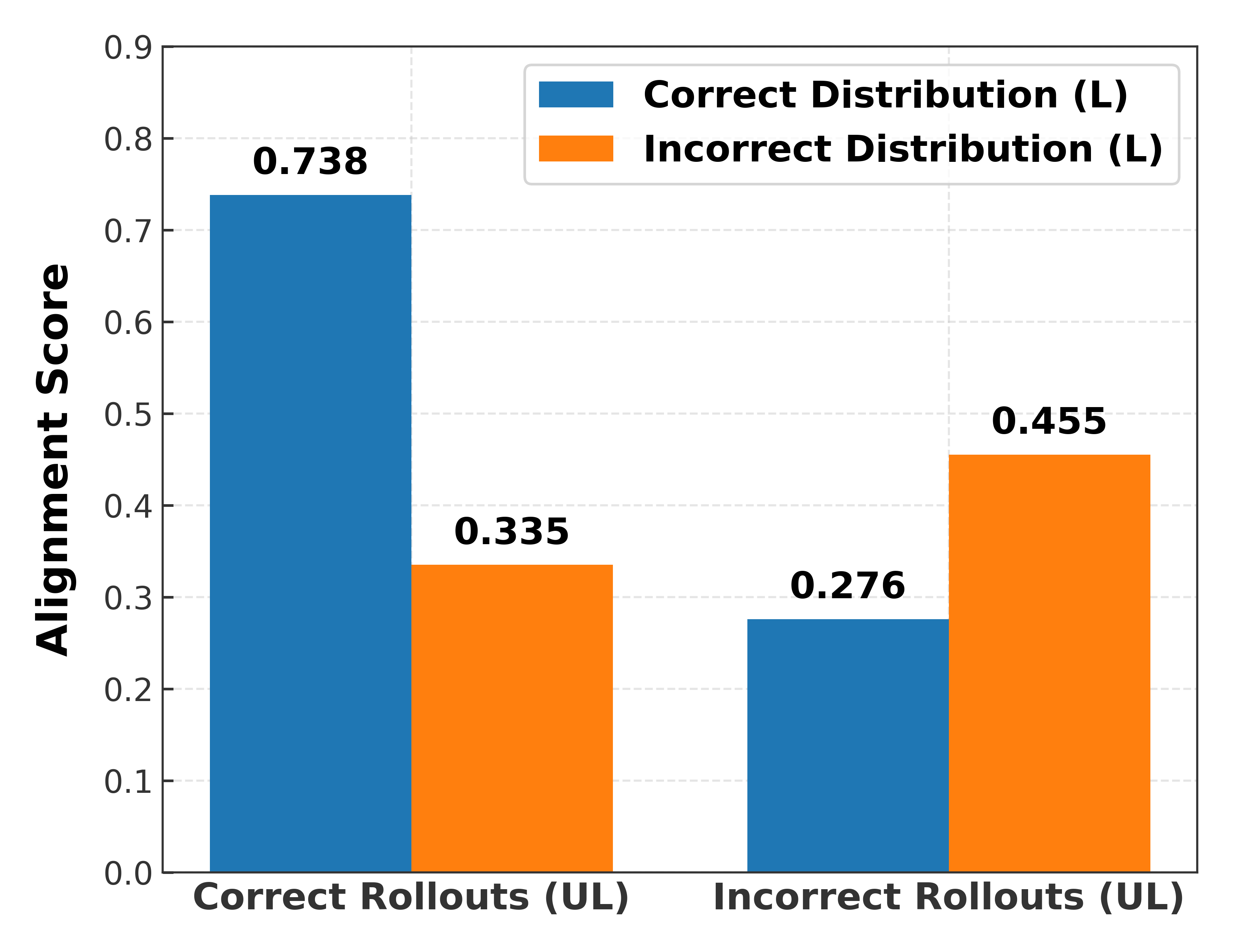}
	}
    \caption{\small 
    (a) TraPO selects a narrow subset, leaving much reliable data underutilized, whereas our method achieves broader, precise coverage for thorough sample mining. 
    (b) Temporal dynamics of distributional separation between correct and incorrect reasoning, which is absent in the base model but sharply emerges during training. 
    (c) Quantification of geometric resonance: unlabeled rollout directions consistently align with their respective labeled distributions conditioned on correctness.}
    \label{fig:cluster_vis}
    \vspace{0.5\baselineskip}
\end{figure*}

To address this, TraPO~\citep{yang2025trapo} prioritizes a semi-supervised RLVR paradigm, which leverages a small set of labeled data to guide the selection of high-quality unlabeled samples, thereby stabilizing training.
While this paradigm paves a promising pathway for annotation-efficient RLVR, our empirical analysis finds that it harvests merely about $12.3\%$ of the reliable unlabeled pool. 
As illustrated in Figure~\ref{fig:low_yield}, a vast majority of valuable instances remain underutilized, thus limiting the performance ceiling.
This recall bottleneck stems from TraPO's restrictive assumption that the performance growth of a reliable unlabeled instance must strictly align with the collective average of labeled data.
Crucially, while the collective average typically ascends monotonically, the optimization trajectory of an individual instance can frequently fluctuate.
Hence, benign individual variance is misinterpreted as unreliability, leading to the premature exclusion of valuable data.

We argue that the core of semi-supervised RLVR lies in decoding the structural discrepancy between correct and incorrect rollouts using labeled data. This learned discrepancy can serve as a robust prior to calibrate the veracity of self-reward signals on unlabeled data, rather than coarsely aligning macro performance as in TraPO. 
To this end, we focus on modeling the global feature distributions of correct and incorrect reasoning. 
Given the ubiquitous normalization in modern LLMs~\citep{zhang2019root,xie2026controlled,fu2026nemotron} that constrains activation scales, our analysis pivots to the directional domain of these representations.
As shown in Figure~\ref{fig:diff_curve}, as training progresses, a clear separation between the directional distributions of correct and incorrect rollouts naturally emerges.
Ultimately, this evolution yields a striking \textbf{geometric resonance}: correct unlabeled rollouts closely align with the correct labeled distribution, whereas incorrect ones consistently gravitate toward their incorrect counterparts (Figure~\ref{fig:dist_match}).
This resonance reveals that global distribution modeling isolates geometric correctness patterns, paving the way for unlabeled sample mining.

In this paper, we propose \textbf{GeoMin}, a two-stage framework that leverages the global distribution geometry to efficiently mine high-quality unlabeled data for semi-supervised RLVR.
The first stage aims to establish robust distribution discriminability via targeted supervised training on labeled data. 
Specifically, normalized hidden states of correct and incorrect rollouts are modeled as distinct von Mises-Fisher (vMF) distributions, where an advantage reweighting mechanism amplifies learning signals for boundary samples to accelerate representation separation.
The second stage extends to semi-supervised training by incorporating unlabeled data. 
Specifically, we compute a geometric confidence score for each unlabeled sample based on its relative affinity to the correct versus incorrect vMF distributions, subsequently fitting a Gaussian Mixture Model (GMM) to adaptively select reliable samples for training.
Through this coordinated two-stage paradigm, GeoMin enables fine-grained confidence calibration and identifies valuable training candidates with high fidelity, achieving an $\textbf{89.0\%}$ F1 score. 
Most strikingly, GeoMin not only outperforms the strongest baseline by $\textbf{+4.1\%}$ ID and $\textbf{+1.7\%}$ OOD absolute accuracy gains, but even \textbf{surpasses fully supervised baseline} using only $\textbf{10\%}$ of the annotations, demonstrating superior performance with minimal annotation overhead.

\begin{figure*}[!t]
    \centering
    \includegraphics[width=\textwidth]{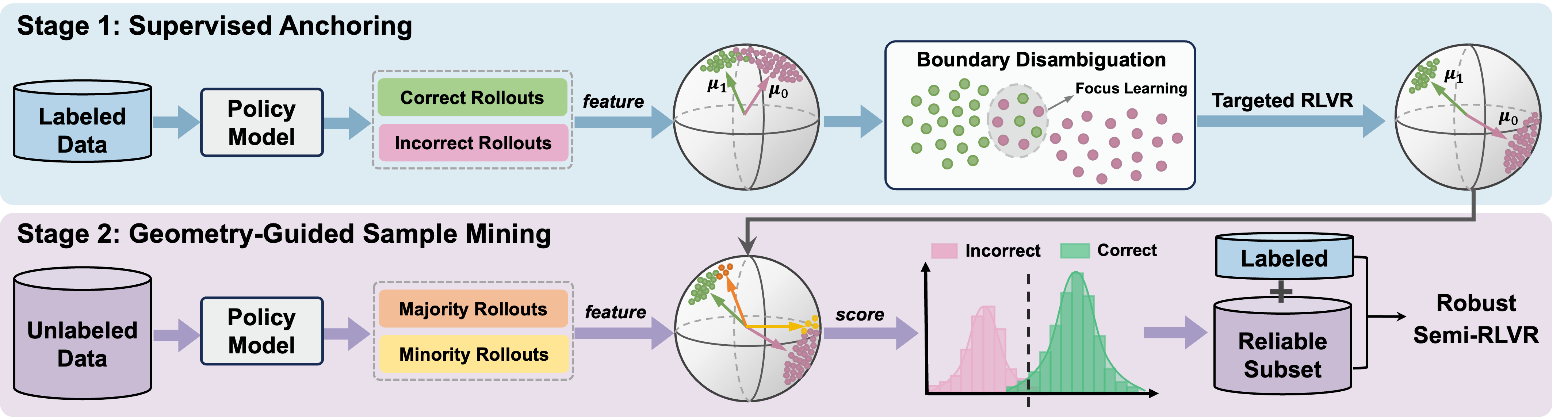}
    \caption{\small Overview of GeoMin. Labeled rollouts are first used to fit vMF distributions and sharpen decision boundaries. Guided by these geometric priors, we evaluate the confidence of unlabeled instances, which are then adaptively filtered via a GMM. Finally, the reliable samples are integrated with the labeled set for robust semi-supervised RLVR training.}
    \label{fig:framework}
\end{figure*}

\section{Preliminary}
\label{sec:preliminary}
\paragraph{Standard RLVR Setup.}
The RLVR paradigm employs a rule-based verifier to assign binary rewards based on response correctness.
Formally, for each question-answer pair $(q,a)$ in the dataset $\mathcal{D}$, the policy $\pi_\theta$ samples a response $y \sim \pi_\theta(\cdot \mid q)$. 
Let $\hat{a}$ denote the answer extracted from $y$; the reward is defined as $R(y,a)=\mathbb{I}[\hat{a}=a]$, where $\mathbb{I}[\cdot]$ is the indicator function. 
For policy optimization, we adopt the \textit{Group Relative Policy Optimization} (GRPO) algorithm~\citep{shao2024deepseekmath}.
GRPO eliminates the value model by estimating advantages directly from a group of responses.
For a given question $q$, we sample $G$ responses $\{y_j\}_{j=1}^G$ from the old policy $\pi_{\theta_{\text{old}}}$ and compute their rewards $R(y_j, a)$.
The group-normalized advantage $\hat{A}_j$ is given by:
\begin{equation}
\label{eq:advantage}
\hat{A}_j = \frac{R(y_j, a) - \text{mean}(\{R(y_j, a)\}_{j=1}^G)}{\text{std}(\{R(y_j, a)\}_{j=1}^G)}.
\end{equation}
% \vspace{-0.5\baselineskip}
Then, the GRPO objective is defined as:
\begin{equation}
\scalebox{0.93}{$
\begin{aligned}
\mathcal{J}_{\text{GRPO}}
& (\theta; \mathcal{D})
= \mathbb{E}[q \sim \mathcal{D}, \{y_j\}_{j=1}^G \sim \pi_{\theta_\text{old}}(\cdot \mid q)] \\
&\hspace{-2em} \frac{1}{G} \! \sum_{j=1}^G \! \frac{1}{|y_j|} \! \sum_{t=1}^{|y_j|} \! \text{CLIP}\bigl(\gamma_{j,t}(\theta), \hat{A}_j, \epsilon\bigr) \! - \! \beta \! \cdot \! \mathbb{D}_{\text{KL}}[\pi_\theta \,\!\|\,\! \pi_{\text{ref}}]
\end{aligned}
$}
\label{eq:GRPO_loss}
\end{equation}
where ${\gamma_{j,t}(\theta) \!\!=\! {\pi_{\theta}(y_{j,t} \!\mid\! q, y_{j,\!<t})}/{\pi_{\theta_{\text{old}}}(y_{j,t} \!\mid\! q, y_{j,\!<t})}}$ is the importance weight, $\text{CLIP}(\gamma, A, \epsilon) = \min[\gamma \!\cdot\! A, \text{clip}(\gamma; 1\!-\!\epsilon, 1\!+\!\epsilon) \!\cdot\! A]$ is the clipped surrogate objective, and $\mathbb{D}_{\text{KL}}$ 
denotes the KL divergence.

\paragraph{Semi-supervised RLVR.} 
To strike a balance between annotation cost and training efficacy, \citet{yang2025trapo} introduces a semi-supervised RLVR paradigm, which jointly utilizes a small labeled set $\mathcal{D}_l = \{(q, a)\}$ and an unlabeled set $\mathcal{D}_u = \{q\}$ through a hybrid reward function:
\begin{equation}
\label{eq:semi_reward}
R_{\text{semi}}(y_j) \!=\!
\begin{cases}
    R(y_j, a), & \!\!\text{if } (q, a) \in \mathcal{D}_l, \\
    R_u(y_j),    & \!\!\text{if } q \in \mathcal{D}_u.
\end{cases}
\end{equation}
For the unlabeled branch, the reward is defined as $R_u(y_j) \!=\! \mathbb{I}[\hat{a}_j \!=\! a^{*}]$, where the pseudo-label $a^{*}$ is derived via majority voting over the $G$ responses, denoted as $a^{*} \!=\! \mathrm{MAJ}(\hat{a}_1, \dots, \hat{a}_G)$.
The resulting rewards are then used to compute the advantage in Eq.~\eqref{eq:advantage} and the GRPO loss in Eq.~\eqref{eq:GRPO_loss}. 
In pursuit of the full potential of this paradigm, our goal is to use the labeled set $\mathcal{D}_l$ to effectively mine unlabeled data with reliable pseudo-labels $a^{*}$ for training.

\section{Method}
In this section, we introduce GeoMin, a two-stage framework designed for semi-supervised RLVR. As illustrated in Figure~\ref{fig:framework}, GeoMin operates sequentially: 
(1) the first stage fits online vMF distributions exclusively on labeled data to establish geometrically discriminative boundaries (Section~\ref{sec:stage1}); 
(2) the second stage leverages these calibrated distributions to adaptively mine high-quality unlabeled samples for training (Section~\ref{sec:stage2}).

\subsection{Formulation of Geometric Priors}
Central to our method is the directional discrepancy between correct and incorrect rollouts, an intrinsic geometric prior for assessing pseudo-label reliability.
To formalize this prior, we explicitly model it as the von Mises-Fisher (vMF) distribution, which is specifically designed for directional data on the hypersphere.
In our context, correct ($c\!=\!1$) and incorrect ($c\!=\!0$) rollouts are characterized as two distinct directional classes. 
Formally, the probability density function for a unit-normalized embedding $\bm{z} \in \mathbb{S}^{d-1}$ belonging to class $c$ is defined as:
\begin{equation}
\label{eq:pdf}
f(\bm{z} | \bm{\mu}_c, \kappa_c) = C_d(\kappa_c) \exp({\kappa_c \bm{\mu}_c^\top \bm{z})}.
\end{equation}
Here, $\bm{\mu}_c$ is the mean direction of the class $c$ with $\Vert\bm{\mu_c}\Vert_2=1$, $\kappa_c \geq 0$ is the class-specific concentration parameter indicating the tightness around $\bm{\mu}_c$, and $C_d(\kappa)$ is the normalization factor whose computation is described in Appendix \ref{sec:norm_compute}.
\begin{proposition}
Let $\bm{z} \in \mathbb{S}^{d-1}$ be the query feature, and let the class-conditional features $\bm{z}_c$ follow a von Mises-Fisher distribution $\text{vMF}(\bm{\mu}_c, \kappa_c)$.
The log-expected kernel density, defined as $\rho(\bm{z}, c) = \log(\mathbb{E}_{\bm{z}_c} [\exp(\bm{z}^\top \bm{z}_c)])$, admits the following closed-form expression:
\begin{equation}
\rho(\bm{z}, c) = \log C_d(\kappa_c) - \log C_d(\kappa_c'),
\end{equation}
where $\kappa_c' = \| \kappa_c \bm{\mu}_c + \bm{z} \|_2$. 
Detailed proof of this proposition is provided in Appendix~\ref{sec:proof_prop_1}.
\end{proposition}
\noindent Fundamentally, $\rho(\bm{z}, c)$ measures the distributional affinity (i.e., directional alignment) of $\bm{z}$ to class $c$. 
This enables us to bridge labeled and unlabeled data by leveraging labeled rollouts to calibrate the vMF priors, which in turn serve as a metric for assessing pseudo-label reliability on unlabeled instances.

\subsection{Stage 1: Supervised Anchoring toward Geometric Discriminability}
\label{sec:stage1}
As shown in Figure~\ref{fig:diff_curve}, a vanilla base model initially fails to exhibit the desired geometric resonance in its representation space. 
Confronted with this cold-start dilemma, incorporating unlabeled data prematurely would inevitably inject massive pseudo-label noise. 
Hence, we deploy a supervised anchoring stage exclusively using the labeled set $\mathcal{D}_l$ to model the spatial distributions and activate this directional alignment capability, paving the way for subsequent unlabeled sample mining.

\paragraph{Distribution Modeling.}
Given that representation spaces vary across network depths, we model dual vMF distributions for each layer $l$ by estimating their corresponding $\bm{\mu}_c^{(l)}$ and $\kappa_c^{(l)}$.
This parameter calibration is performed in an online, batch-wise manner to capture the dynamic evolution of the representation space during training.
Specifically, at each training step, we are provided with a mini-batch of question-answer pairs $\mathcal{B}={\{(q_i,a_i)\}_{i=1}^N}$, where each question $q_i$ is paired with $G$ rollouts $\mathcal{R}_i \!=\! \{y_{i,j}\}_{j=1}^G$.
These rollouts are flattened into a unified batch pool $\mathcal{R}_{\text{all}} \!=\! \bigcup_{i=1}^N \!\mathcal{R}_i \!=\! \{y_k\}_{k=1}^M$ with $M \!=\! N \!\cdot\! G$.
For each rollout $y_k \!\in\! \mathcal{R}_{\text{all}}$, we extract the hidden state of its last token at the $l$-th layer, denoted as $\bm{h}_k^{(l)} \!\in\! \mathbb{R}^d$, which encapsulates the cumulative sequence semantics.
We then project this state onto the unit hypersphere to obtain the normalized feature $\bm{z}_k^{(l)} \!=\! \bm{h}_k^{(l)} / \|\bm{h}_k^{(l)}\|_2 \!\in\! \mathbb{S}^{d-1}$. 
For each class $c$, the mean direction vector $\bar{\bm{z}}_c^{(l)}$ is updated via a momentum-based moving average (EMA):
\begin{equation}
\bar{\bm{z}}_c^{(l)} \leftarrow \lambda \bar{\bm{z}}_c^{(l)} + (1 - \lambda) \frac{\sum_{k=1}^M \mathbb{I}(c_k = c) \bm{z}_k^{(l)}}{\sum_{k=1}^M \mathbb{I}(c_k = c)},
\label{eq:mean_vec}
\end{equation}
where $\lambda \!\in\! (0,1)$ is the momentum coefficient, and $c_k = \mathbb{I}(\hat{a}_k=a_i)$ denotes the correctness of rollout $y_k \in \mathcal{R}_i$. The normalized mean direction is subsequently defined as $\bm{\mu}_c^{(l)} = \bar{\bm{z}}_c^{(l)} / \|\bar{\bm{z}}_c^{(l)}\|_2$.
Following \citet{sra2012short}, the layer-wise concentration parameter $\kappa_c^{(l)}$ can be directly estimated in closed form via:
\begin{equation}
\kappa_c^{(l)} = \frac{\|\bar{\bm{z}}_c^{(l)}\|_2 (d - \|\bar{\bm{z}}_c^{(l)}\|_2^2)}{1 - \|\bar{\bm{z}}_c^{(l)}\|_2^2}.
\label{eq:mu_kappa}
\end{equation}
With these layer-wise distributions fitted online, the affinity between an arbitrary feature $\bm{z}^{(l)}$ and the correctness class $c$ can be instantly measured via:
\begin{equation}
\rho(\bm{z}^{(l)}, c) = \log C_d(\kappa_c^{(l)}) - \log C_d(\kappa_c'^{(l)}),
\end{equation}
where $\kappa_c'^{(l)} = \| \kappa_c^{(l)} \bm{\mu}_c^{(l)} + \bm{z}^{(l)} \|_2$. 

\paragraph{Boundary Disambiguation.}
To accelerate the emergence of geometric resonance, we propose to actively identify geometrically confounded instances, and subsequently focus the model's optimization on these pivotal boundary samples.
Concretely, for each question $q_i$, we partition its rollout set $\mathcal{R}_i$ into a correct subset $\mathcal{R}_i^+ \!=\! \{y_k \!\in\! \mathcal{R}_i \!\mid\! c_k \!=\! 1\}$ and an incorrect subset $\mathcal{R}_i^- \!=\! \{y_k \!\in\! \mathcal{R}_i \!\mid\! c_k \!=\! 0\}$.
Based on the relative spatial configuration of the rollouts, a question $q_i$ is defined as an ambiguous boundary sample if either of the following geometric inversions occurs:
\begin{equation}
\label{eq:boundary_cond}
\scalebox{0.95}{$
\begin{aligned}
\!\max_{y_k \in \mathcal{R}_i^-} \sum_{l=1}^L \rho(\bm{z}_k^{(l)}, 1) &> \min_{y_k \in \mathcal{R}_i^+} \sum_{l=1}^L \rho(\bm{z}_k^{(l)}, 1), \\
\!\text{or} \, \max_{y_k \in \mathcal{R}_i^+} \sum_{l=1}^L \rho(\bm{z}_k^{(l)}, 0) &> \min_{y_k \in \mathcal{R}_i^-} \sum_{l=1}^L \rho(\bm{z}_k^{(l)}, 0).
\end{aligned}
$}
\end{equation}
where $L$ is the total number of hidden layers.
In other words, correct rollouts should consistently exhibit closer proximity to the correct direction than incorrect ones, and vice versa.
Instances satisfying Eq.~\eqref{eq:boundary_cond} deviate from this expectation, indicating that their representations remain ambiguous and stay on the model's cognitive boundary.
We argue that properly learning these boundary samples is the key to eliciting geometric resonance.
Therefore, we amplify their advantages as follows:
% \begin{equation}
% \hat{A}_k \leftarrow \left(1 + (\alpha - 1) \cdot \mathbb{I}[q_i \in \mathcal{B}]\right) \cdot \hat{A}_k \quad \text{for} \quad y_k \!\in\! \mathcal{R}_i
% \label{eq:advantage_amplify}
% \end{equation}
\begin{equation}
\tilde{A}_k \!=\! \left(1 + \alpha \cdot \mathbb{I}[q_i \!\in\! \Omega]\right) \!\cdot\! \hat{A}_k, \, \text{where } y_k \!\in\! \mathcal{R}_i
\label{eq:advantage_amplify}
\end{equation}
where $\Omega$ denotes the set of questions satisfying Eq.~\eqref{eq:boundary_cond}, $\hat{A}_k$ is the group-normalized advantage of rollout $y_k$ computed via Eq.~\eqref{eq:advantage}, and $\alpha \!>\! 0$ is a constant scaling factor. 
The modified advantage $\tilde{A}_k$ is then employed in Eq.~\eqref{eq:GRPO_loss} for policy optimization.

\subsection{Stage 2: Semi-Supervised RLVR with Geometry-Guided Sample Mining}
\label{sec:stage2}
Following the supervised anchoring stage, we perform semi-supervised RLVR training on a unified mini-batch $\mathcal{B} = \mathcal{B}_l \cup \mathcal{B}_u$, where $\mathcal{B}_l \!=\! \{(q_i, a_i)\}_{i=1}^{N_l}$ and $\mathcal{B}_u \!=\! \{q_i\}_{i=1}^{N_u}$ denote the labeled and unlabeled subsets, respectively. 
With the activated geometric resonance, we leverage the dual vMF distributions calibrated on $\mathcal{B}_l$ to rigorously evaluate pseudo-label reliability on $\mathcal{B}_u$, thereby mining high-quality unlabeled instances for robust training.

\paragraph{Continuous vMF Adaptation.}
To align with the shifting representation space as the policy evolves, we continuously refresh the layer-wise vMF distributions $(\bm{\mu}_c^{(l)}, \kappa_c^{(l)})$ at each training step. Following the online calibration detailed in Eq.~\eqref{eq:mean_vec} and Eq.~\eqref{eq:mu_kappa}, this parameter update is performed exclusively using the labeled subset $\mathcal{B}_l$  to establish a stable yet time-varying geometric anchor.

\paragraph{Discriminative Layer Selection.}
As illustrated in Figure~\ref{fig:diff_curve}, the capacity to differentiate correctness exhibits substantial variation across different network depths. To quantify this effect and isolate the layers with the highest diagnostic fidelity for subsequent reliability assessment, we define the \textit{layer separability} $\Delta^{(l)}$ over the labeled rollouts as:
\begin{equation}
\label{eq:separability}
\scalebox{0.92}{$
\begin{aligned}
\Delta^{(l)} =& \frac{1}{|\mathcal{H}_t^+|} \sum_{y_k \in \mathcal{H}_t^+} \! \left( \rho(\bm{z}_k^{(l)}\!, 1) - \rho(\bm{z}_k^{(l)}\!, 0) \right) \\
+& \frac{1}{|\mathcal{H}_t^-|} \sum_{y_k \in \mathcal{H}_t^-} \! \left( \rho(\bm{z}_k^{(l)}\!, 0) - \rho(\bm{z}_k^{(l)}\!, 1) \right),
\end{aligned}
$}
\end{equation}
where $\mathcal{H}_t^+$ and $\mathcal{H}_t^-$ denote the historical buffers of correct and incorrect labeled rollouts aggregated up to step $t$, respectively. 
Based on this metric, we retain the top-$K$ layers with the highest separability, denoted as $\mathcal{I}$, for pseudo-label verification.

\paragraph{Unlabeled Sample Mining.}
For each unlabeled instance $q_i \in \mathcal{B}_u$, we first derive its pseudo-label $a_i^*$ via majority voting. This voting naturally partitions its generated rollout pool $\mathcal{R}_i$ into a pseudo-correct majority subset $\mathcal{R}_i^* = \{y_k \in \mathcal{R}_i \mid \hat{a}_k = a_i^*\}$ and a pseudo-incorrect minority subset $\mathcal{R}_i \setminus \mathcal{R}_i^*$. 
To evaluate whether this pseudo-label aligns with the model's underlying geometric prior, we calculate a confidence score $s_i$ by aggregating the directional affinities over the selected discriminative layers $\mathcal{I}$:
\begin{equation}
\label{eq:conf_score}
\scalebox{0.92}{$
\begin{aligned}
s_i = \frac{1}{|\mathcal{R}_i|} \sum_{y_k \in \mathcal{R}_i} \sum_{l \in \mathcal{I}} \sigma_k\! \left( \rho(\bm{z}_k^{(l)}\!, 1) \!-\! \rho(\bm{z}_k^{(l)}\!, 0) \right),
\end{aligned}
$}
\end{equation}
where $\sigma_k \!=\! 1$ if $y_k \!\in\! \mathcal{R}_i^*$, and $\sigma_k \!=\! -1$ otherwise. Fundamentally, Eq.~\eqref{eq:conf_score} regularizes the pseudo-label: it assigns a high confidence score only when the majority rollouts strongly align with the correct vMF prior while the minority rollouts concurrently align with the incorrect prior.
However, manually calibrating a static threshold across these evolving scores is highly suboptimal and intensive.
To achieve adaptive and automated sample selection, we instead fit a two-component Gaussian Mixture Model (GMM) over the dynamic score pool $\mathcal{S} = \{s_i \!\mid\! q_i \!\in\! \mathcal{B}_u\}$.
Specifically, let $\mathcal{C}_{\text{high}}$ denote the Gaussian component with the larger mean. 
For each unlabeled instance, we compute its posterior probability $w_i \!=\! p(\mathcal{C}_{\text{high}} \!\mid\! s_i)$ via the Expectation-Maximization algorithm, which explicitly reflects the likelihood that the pseudo-label $a_i^*$ is correct. 
We then identify the reliable pseudo-labeled samples as $\mathcal{B}_u^* = \{ (q_i, a_i^*) \mid w_i > \tau \}$, which are finally merged with the labeled subset as $\mathcal{B}_l \cup \mathcal{B}_u^*$ for robust semi-supervised RLVR as described in Section~\ref{sec:preliminary}.

\section{Experiments}
In this section, we present the main results and a detailed analysis showing that GeoMin achieves superior performance with minimal annotation overhead. More experimental details and results are provided in Appendix~\ref{app:exp_setup} and \ref{app:exp_results}, respectively.

\subsection{Setup}
\label{sec:setup}
\paragraph{Implementation Details.}
Our algorithm is built upon the \textit{verl} framework~\citep{sheng2025hybridflow}, following the standard GRPO recipe~\citep{shao2024deepseekmath}. Training is conducted on the challenging subset of Deepmath-103k~\citep{he2025deepmath} (difficulty $\geq 6$) using $8 \!\times\! \text{A100}$ GPUs, with a total batch size of $128$, a micro-batch size of $32$, and a learning rate of $1e^{-6}$.
The number of rollouts per prompt is set to $G=8$. 
For the first stage, the scaling factor $\alpha$ in Eq.~\eqref{eq:advantage_amplify} is fixed at $2$, and the momentum $\lambda$ in Eq.~\eqref{eq:mean_vec} is adaptively updated based on the ratio of current batch samples to cumulative historical instances.
For the second stage, we select the top $K=10\%$ discriminative layers for confidence score computation in Eq.~\eqref{eq:conf_score} and set the GMM verification threshold to $\tau=0.5$.
We adopt Qwen3-8B-Base~\citep{yang2025qwen3} as the default backbone; evaluations on more models of diverse families and scales are provided in Appendix~\ref{app:more_model}.

\begin{table*}[!t]
\centering
\caption{In-domain (ID) and out-of-domain (OOD) performance using Qwen3-8B-Base. Methods are evaluated under purely unsupervised and/or semi-supervised ($10\%$ labeled data) settings. \textbf{Bold} denotes the best results.}
\label{tab:main_results}
\setlength{\tabcolsep}{3pt}  
\renewcommand{\arraystretch}{1.2} 
\resizebox{\textwidth}{!}{%
\begin{tabular}{lcccccc|cccc}
\toprule
\multirow{2}{*}{\textbf{Methods}} & \multicolumn{6}{c}{\textbf{In-Domain Performance}} & \multicolumn{4}{c}{\textbf{Out-of-Domain Performance}} \\
\cmidrule(lr){2-7} \cmidrule(lr){8-11}
 & \textbf{AIME 24/25} & \textbf{AMC} & \textbf{MATH-500} & \textbf{Minerva} & \textbf{Olympiad} & \textbf{Avg.} & \textbf{ARC-c} & \textbf{GPQA}$^{*}$ & \textbf{MMLU-Pro} & \textbf{Avg.} \\
\midrule
Before RL & 9.4/7.3 & 38.0 & 64.4 & 25.6 & 32.5 & 29.5 & 29.7 & 11.5 & 46.6 & 29.3 \\
\midrule
\rowcolor{Gray}\multicolumn{11}{c}{Unsupervised RLVR Training w/o Any Labeled Samples} \\
\midrule
TTRL & 17.7/12.5 & 54.4 & 81.7 & 43.8 & 46.0 & 42.7 & 92.3 & 44.6 & 64.0 & 67.0 \\
Tok-entropy & 12.3/11.1 & 44.8 & 78.5 & 34.6 & 41.8 & 37.2 & 43.1 & 10.2 & 56.7 & 36.7 \\
Seq-entropy & 14.2/15.8 & 51.9 & 78.3 & 38.0 & 43.1 & 40.2 & 81.3 & 20.8 & 59.5 & 53.9 \\
Self-certainty & 9.1/7.7 & 45.2 & 77.1 & 41.1 & 39.7 & 36.6 & 64.4 & 18.9 & 59.2 & 47.5 \\
Co-rewarding & 14.4/12.2 & 52.3 & 79.2 & 43.1 & 44.4 & 40.9 & 91.0 & 27.8 & 62.7 & 60.5 \\
\midrule
\rowcolor{Gray}\multicolumn{11}{c}{Semi-Supervised RLVR Training with $10\%$ Labeled Samples} \\
\midrule
TTRL & 16.7/16.1 & 54.0 & 80.7 & 44.2 & 47.6 & 43.2 & 92.9 & 46.1 & 64.3 & 67.8 \\
Tok-entropy & 15.5/14.6 & 46.5 & 79.0 & 37.0 & 42.3 & 39.2 & 73.9 & 17.8 & 57.4 & 49.7 \\
Seq-entropy & 15.1/16.6 & 51.0 & 79.4 & 39.0 & 42.9 & 40.7 & 78.1 & 17.3 & 58.8 & 51.4 \\
Self-certainty & 11.1/9.4 & 44.1 & 77.5 & 42.4 & 39.9 & 37.4 & 70.8 & 22.9 & 59.7 & 51.1 \\
Co-rewarding & 16.1/13.8 & 52.6 & 79.7 & 42.7 & 44.6 & 41.6 & 91.7 & 31.2 & 63.4 & 62.1 \\
TraPO & 19.8/14.3 & 56.5 & 81.5 & 43.4 & 47.4 & 43.8 & 92.3 & 45.8 & 63.5 & 67.2 \\
\textbf{GeoMin (ours)} & \textbf{24.5}/\textbf{19.2} & \textbf{58.5} & \textbf{86.7} & \textbf{45.5} & \textbf{52.7} & \textbf{47.9} & \textbf{93.7} & \textbf{48.2} & \textbf{66.5} & \textbf{69.5} \\
\midrule
\textcolor{gray}{{Fully Supervised}} & \textcolor{gray}{23.1/19.8} & \textcolor{gray}{55.1} & \textcolor{gray}{85.0} & \textcolor{gray}{46.5} & \textcolor{gray}{50.9} & \textcolor{gray}{46.7} & \textcolor{gray}{93.9} & \textcolor{gray}{50.5} & \textcolor{gray}{65.5} & \textcolor{gray}{70.0} \\
\bottomrule
% \vspace{-35pt}
\end{tabular}%
}
\end{table*}

\paragraph{Evaluation.}
Following the standard protocols in prior work \citep{yan2025learning,yang2025trapo}, we conduct comprehensive evaluations across diverse benchmarks spanning mathematical and general reasoning. 
For the mathematical domain, we test on six competition-level datasets: AIME 2024, AIME 2025, MATH-500~\citep{hendrycks2021measuring}, Minerva~\citep{lewkowycz2022solving}, AMC~\citep{li2024numinamath}, and OlympiadBench~\citep{he2024olympiadbench}. 
To assess out-of-domain generalization, we further include three general reasoning benchmarks: ARC-c~\citep{clark2018think}, GPQA-diamond~\citep{rein2024gpqa} (denoted GPQA$^*$), and MMLU-Pro~\citep{wang2024mmlu}. 
Given variations in test-set sizes, we report $\text{avg}@32$ for AIME 2024/2025 and AMC, $\text{pass}@1$ for MMLU-Pro, and $\text{avg}@4$ for others. 
All evaluations use temperature $0.6$ and top-$p$ $1.0$.

\paragraph{Baselines.} 
We compare GeoMin with six weakly-supervised RLVR baselines: 
(1) \textbf{TTRL} \citep{zuo2025ttrl}, which rewards rollouts yielding the majority answer;
(2) \textbf{Tok-entropy} \citep{agarwal2025unreasonable}, which rewards rollouts with lower average token entropy;
(3) \textbf{Seq-entropy} \citep{agarwal2025unreasonable}, which rewards rollouts with higher sequence probability;
(4) \textbf{Self-certainty} \citep{zhao2025learning}, which rewards rollouts with a higher KL divergence of token distributions from the uniform distribution;
(5) \textbf{Co-rewarding} \citep{zhang2025co}, which rewards rollouts based on pseudo-labels generated by a slowly-updated reference teacher;
(6) \textbf{TraPO} \citep{yang2025trapo}, a semi-supervised baseline that selects unlabeled data whose pass-rate trajectories align closest with those of labeled samples.
Baselines (1)–(5) are evaluated under a purely unsupervised setting, and further extended alongside (6) to a semi-supervised regime, where labeled samples are rewarded by correctness, while unlabeled ones rely on their respective self-guided rewards.

\begin{figure*}[ht]
	\centering
    \subfigure[Annotation Scaling\label{fig:label_cost}]{
		\includegraphics[width=0.372\textwidth]{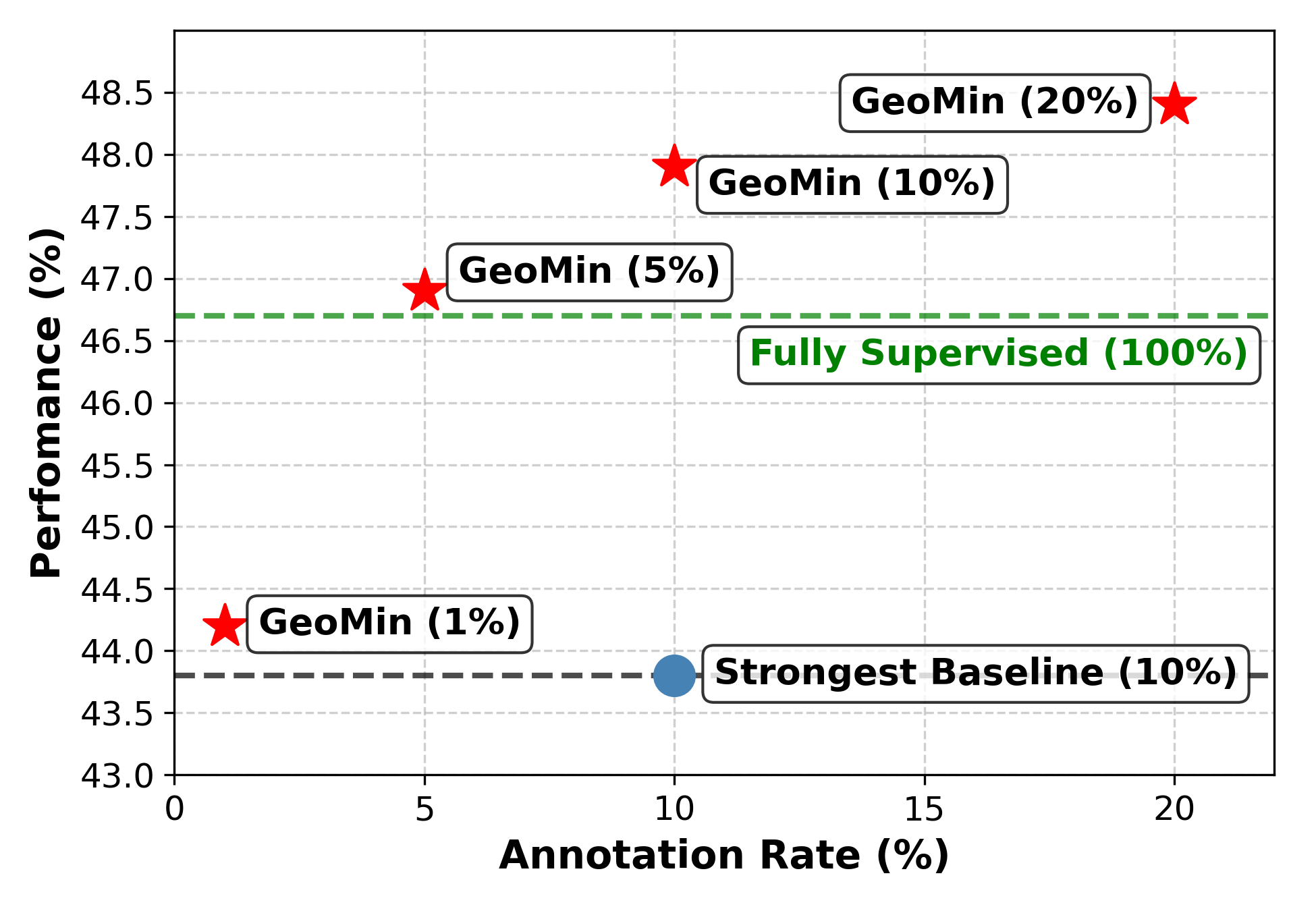}
	}
    \subfigure[Unlabeled Utilization\label{fig:f1_compare}]{
		\includegraphics[width=0.265\textwidth]{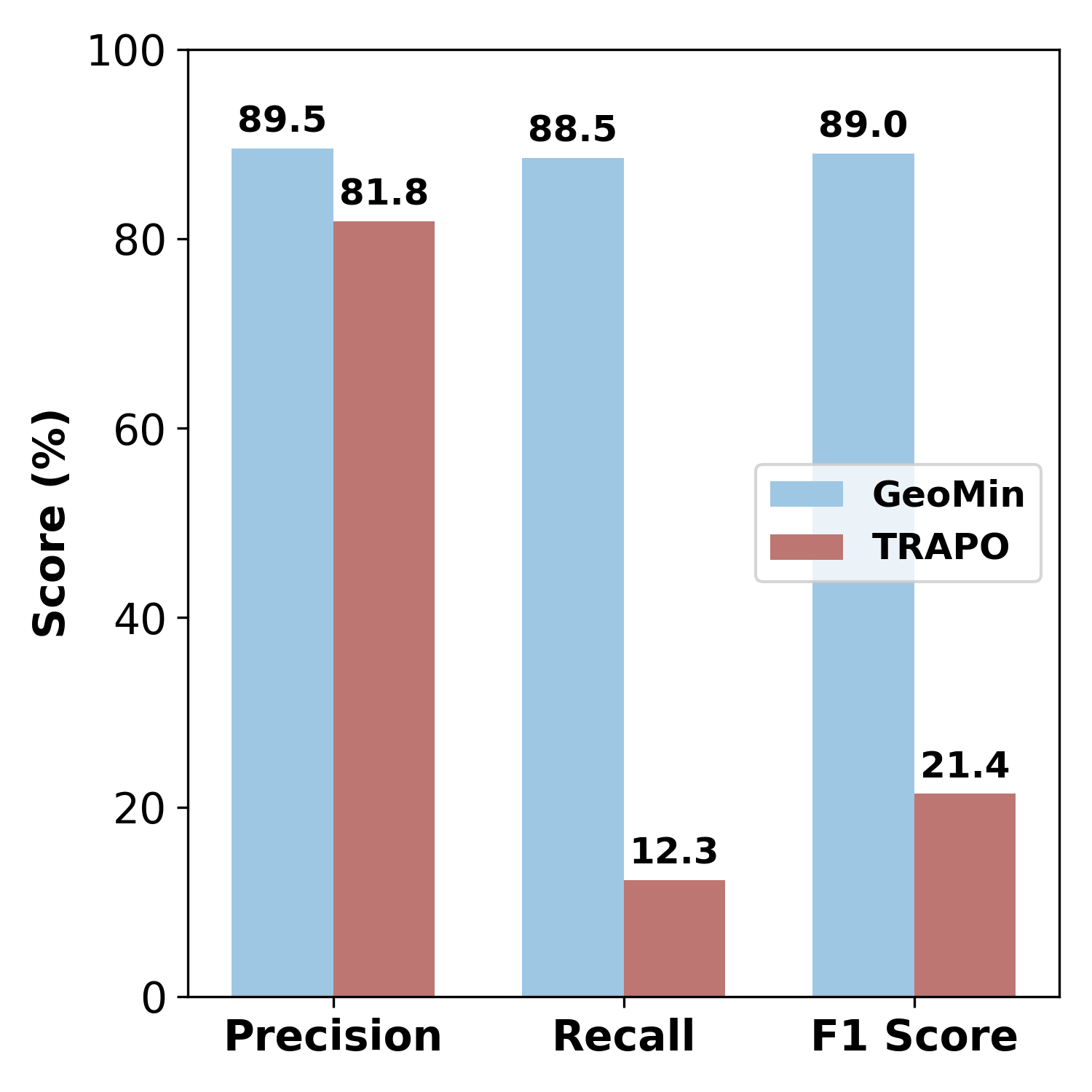}
	}
    \subfigure[Component Ablation\label{fig:key_ablation}]{
		\includegraphics[width=0.315\textwidth]{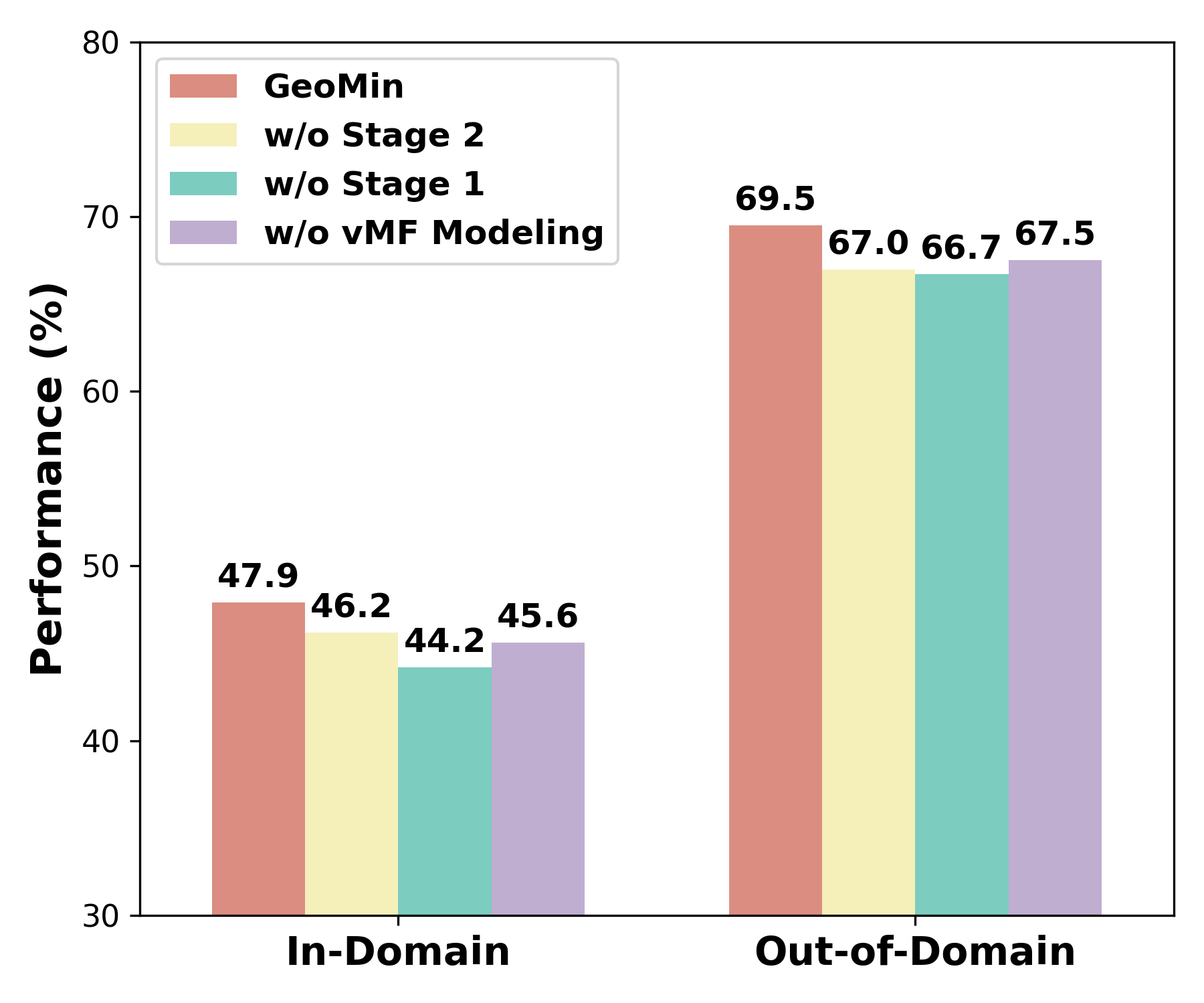}
	}
   \caption{\small (a) Performance (ID) of GeoMin across varying annotation rates. 
   (b) Precision, recall, and F1 score calculated on the reliable unlabeled samples selected by TraPO and GeoMin. 
   (c) Key component ablation study on ID and OOD tasks.}
   \label{fig:analysis_1}
\end{figure*}

\subsection{Main Results}
As shown in Table~\ref{tab:main_results}, GeoMin consistently outperforms all baselines with absolute gains of $+4.1\%$ and $+1.7\%$ on ID and OOD average performance over the strongest competitors.
Notably, TraPO exhibits marginal advantages over the vanilla TTRL baseline; it only yields a minor $+0.6\%$ gain on ID average while suffering a $-0.6\%$ drop on OOD average.
This can be largely attributed to its underutilization of unlabeled data, which severely caps its scaling potential and performance ceiling.
In contrast, GeoMin precisely and comprehensively unearths the wealth of unlabeled data via geometric resonance, enabling the semi-supervised RLVR paradigm to truly live up to its promise.

Most strikingly, with a mere $10\%$ annotation rate, GeoMin outperforms the fully-supervised baseline on ID performance by $+1.2\%$, while maintaining competitive OOD performance ($69.5\%$ vs. $70.0\%$). 
We attribute this to the robust representation space learned through the synergy between our two training stages, which contrasts with the indiscriminate fitting of full supervision. 
Specifically, Stage 1 focuses optimization on confounded boundary samples to promote representation separation and clear up decision boundaries. 
Stage 2 selects unlabeled instances that match the calibrated distributions to progressively refine and enrich the representation space. 
Ultimately, by combining initial boundary separation with sequential distribution refinement, GeoMin internalizes intrinsic logical structures, leading to superior generalization with exceptional annotation efficiency.

\subsection{Further Analysis}
\paragraph{Efficacy in Annotation Mitigation.}
To quantify the efficacy of GeoMin in minimizing annotation dependency, we analyze its performance scaling behaviors under varying label ratios. 
As illustrated in Figure~\ref{fig:label_cost}, GeoMin exhibits a consistent and monotonic performance scaling as the label ratio increases, reaching $44.2\%$, $46.9\%$, $47.9\%$, and $48.4\%$ accuracy at $1\%$, $5\%$, $10\%$, and $20\%$ annotations, respectively. 
This steady upward trend underscores the scaling potential and robustness of our framework.
Moreover, GeoMin delivers remarkable data efficiency: a mere $1\%$ of annotations already surpasses the strongest baseline TraPO with $10\%$ labels, while only $5\%$ of annotations are required to match the fully supervised baseline utilizing $100\%$ labels.
These results demonstrate GeoMin's capacity to bypass expensive annotation barriers without compromising and even boosting the optimization performance.

\begin{figure*}[ht]
	\centering
    \subfigure[Step 0\label{fig:vmf_init}]{
		\includegraphics[width=0.248\textwidth]{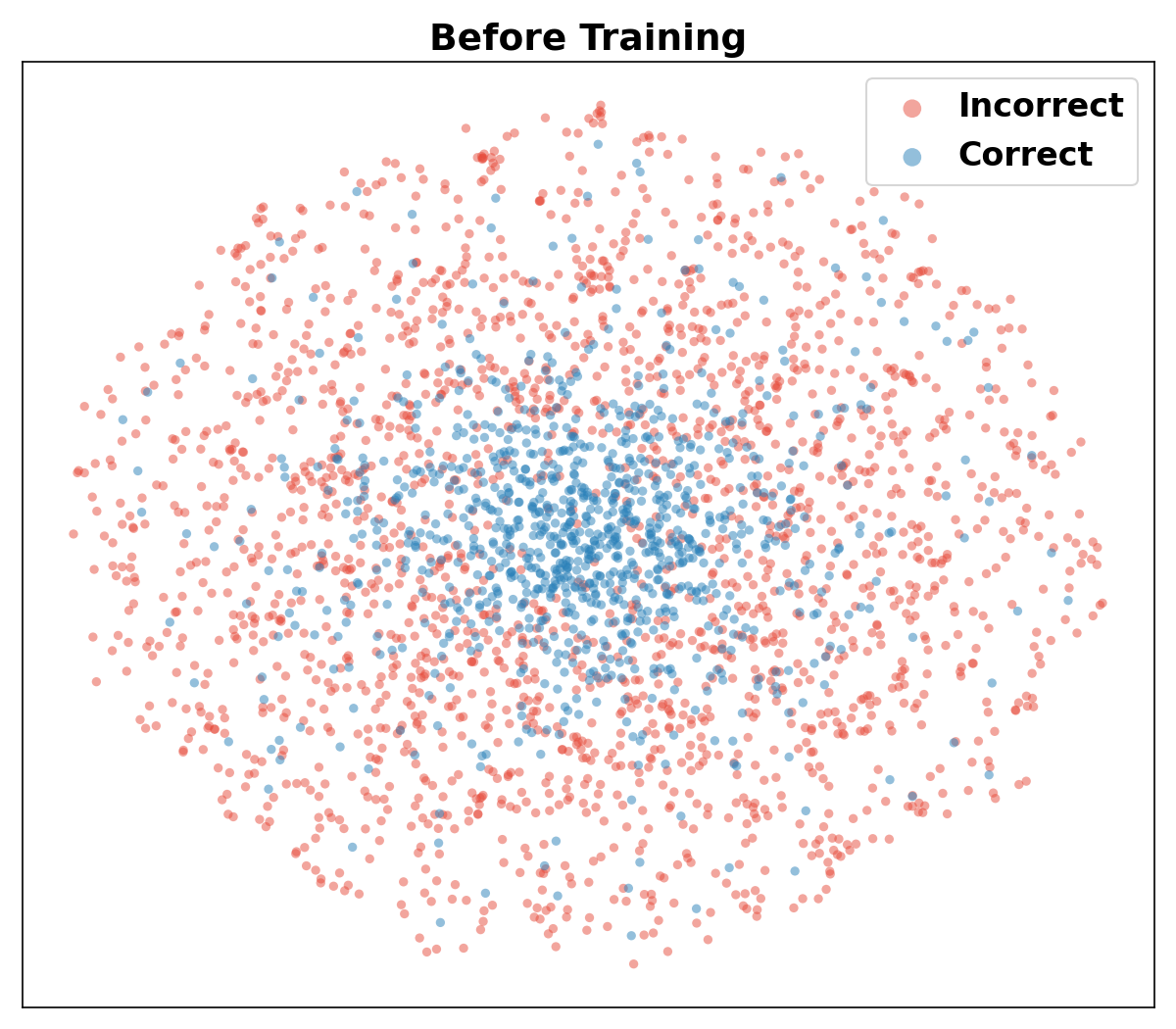}
	}
    \subfigure[Step 100 (w/o BD)\label{fig:vmf_wo_bd}]{
		\includegraphics[width=0.248\textwidth]{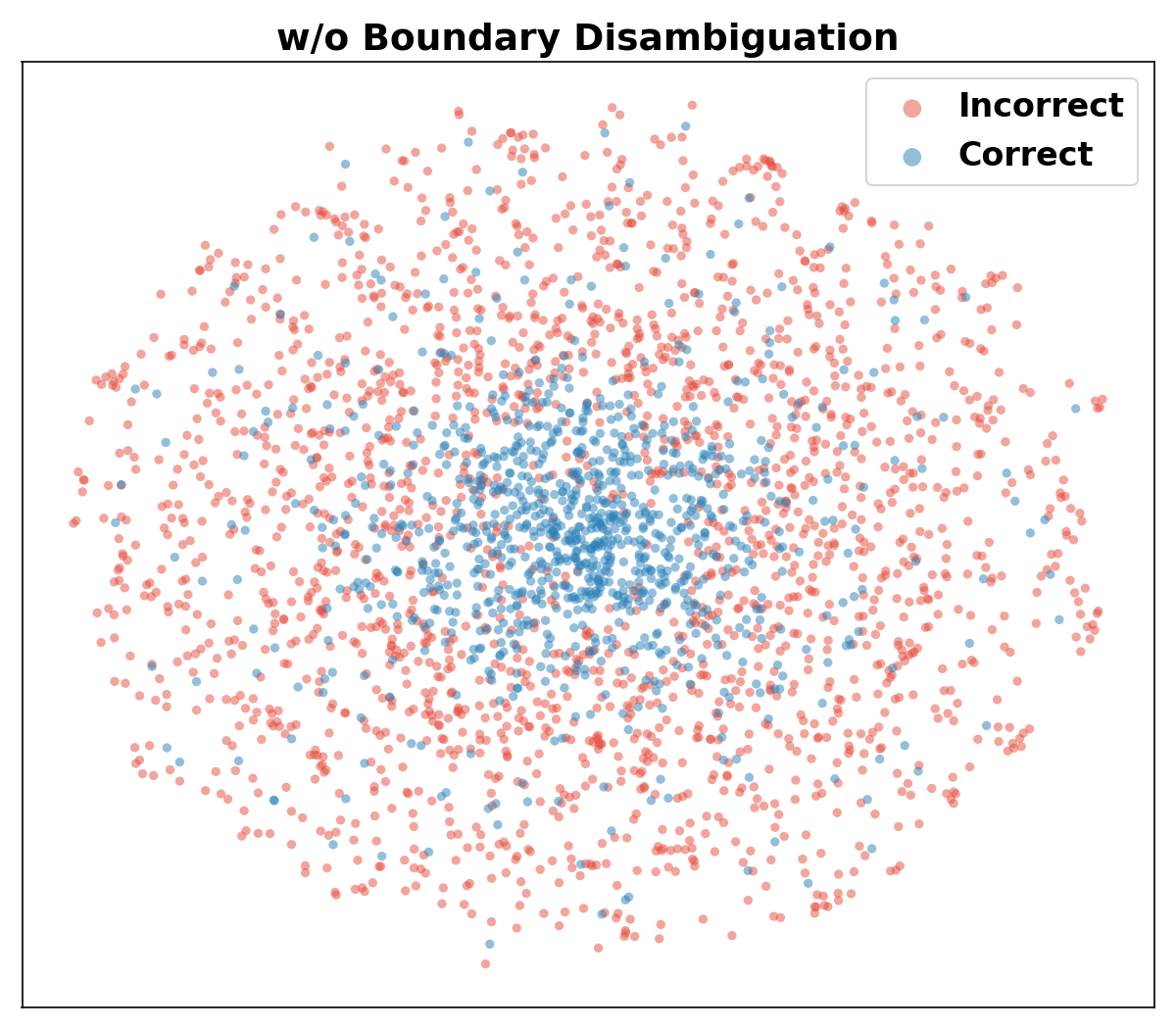}
	}
    \subfigure[Step 100 (w/ BD)\label{fig:vmf_w_bd}]{
		\includegraphics[width=0.248\textwidth]{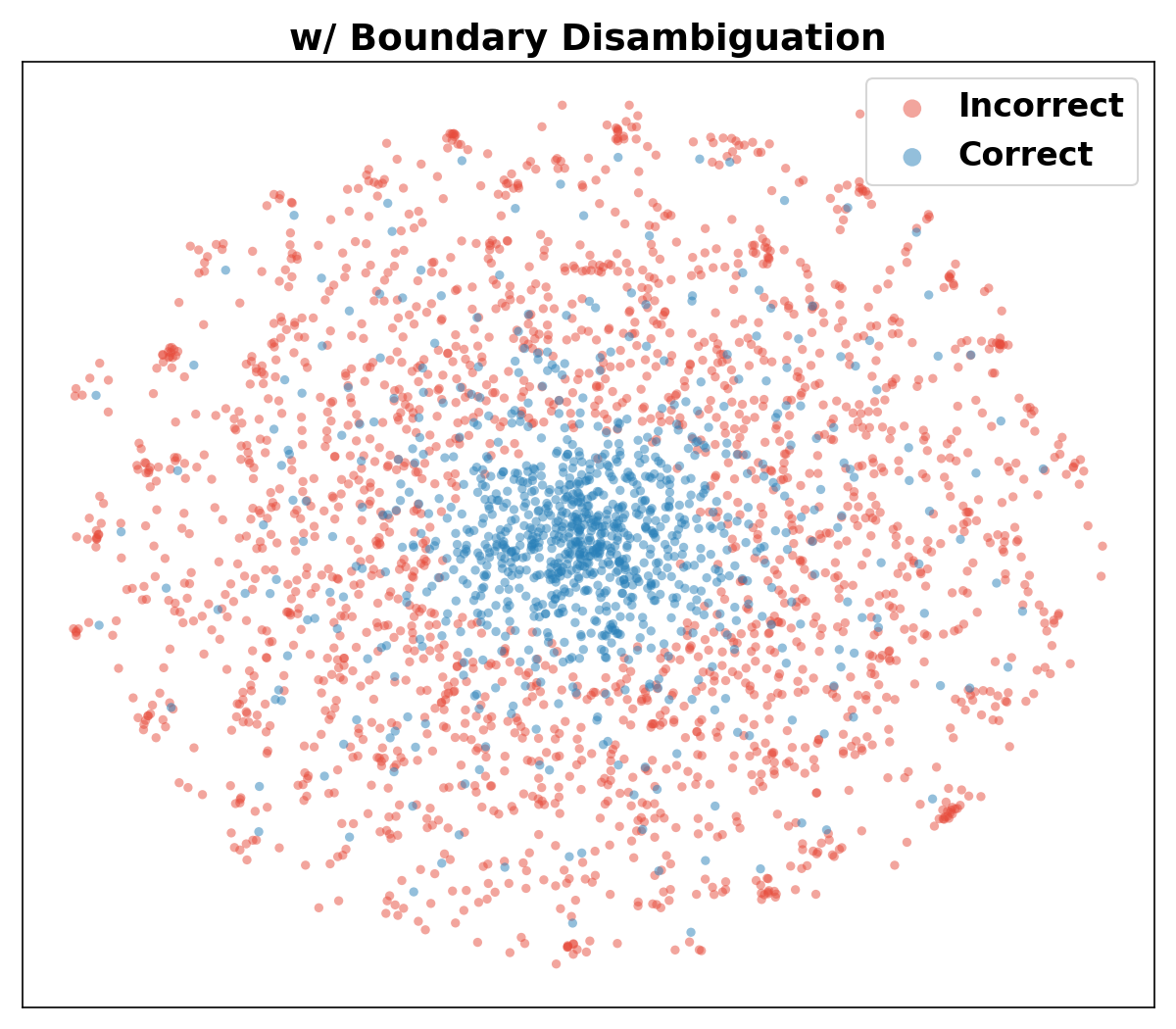}
	}
    \subfigure[Time Profiling\label{fig:timing}]{
		\includegraphics[width=0.191\textwidth]{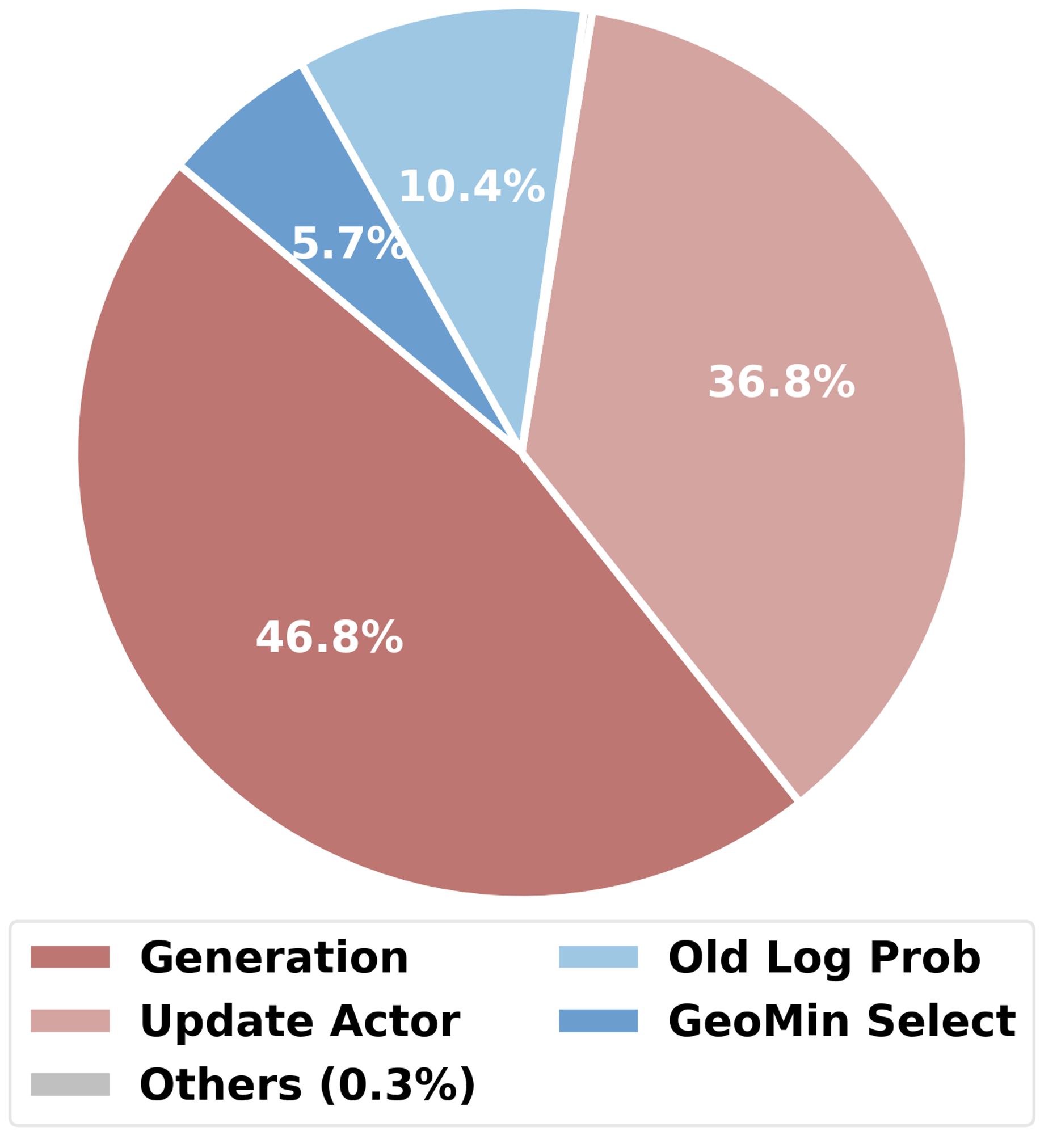}
	}
   \caption{\small (a--c) T-SNE visualizations of vMF distributions for correct and incorrect rollouts across different stages (initial status, 100 steps with/without boundary disambiguation). (d) Training time allocation across different operational phases.}
   \label{fig:analysis_2}
\end{figure*}

\paragraph{Efficacy in Unlabeled Data Mining.}
To deeply understand how GeoMin unlocks the potential of unlabeled data, we evaluate the quality of selected pseudo-labels and compare them against TraPO.
As shown in Figure~\ref{fig:f1_compare}, TraPO suffers from a severe recall bottleneck, achieving a precision of $81.8\%$ but a recall of only $12.3\%$ ($21.4\%$ F1 score). 
This confirms that its strict trajectory-matching mechanism overlooks abundant reliable samples evolving along alternative paths. 
In sharp contrast, GeoMin effectively breaks this limitation, achieving a precision of $89.5\%$ alongside a phenomenal recall of $88.5\%$, culminating in an F1 score of $89.0\%$. 
These results indicate that GeoMin establishes a more fundamental connection between labeled and unlabeled data through distributional modeling, significantly enhancing unlabeled sample utilization without sacrificing pseudo-label quality.

\paragraph{Efficacy in Boundary Disambiguation.}
To intuitively verify our boundary disambiguation mechanism, Figure~\ref{fig:analysis_2} presents the T-SNE visualizations of vMF distributions for correct and incorrect rollouts. 
Initially (Figure~\ref{fig:vmf_init}), the embeddings are intertwined, reflecting an ambiguous decision boundary. 
After $100$ optimization steps \textit{w/o Boundary Disambiguation} (Figure~\ref{fig:vmf_wo_bd}), the two distributions still exhibit substantial overlap despite a marginal contraction of correct samples. Conversely, with \textit{Boundary Disambiguation} (Figure~\ref{fig:vmf_w_bd}), they become noticeably more distinguished, tending towards a well-defined decision boundary. 
This contrast validates our design: by leveraging advantage reweighting, GeoMin prioritizes learning critical boundary samples, sharpening the discriminative boundary in the underlying representation space.

\paragraph{Key Component Ablation.}
We perform a comprehensive ablation study to validate the design of GeoMin (Figure~\ref{fig:key_ablation}). 
First, \textit{w/o Stage 2} shows limited efficacy, dropping by $1.7\%$ ID and $2.5\%$ OOD performance, which confirms the necessity of mining and effectively exploiting unlabeled data.
Moreover, \textit{w/o Stage 1} incurs the sharpest performance drop, with ID and OOD accuracies plunging by $3.7\%$ and $2.8\%$, respectively.
This indicates that incorporating unlabeled samples without establishing geometric discriminability introduces highly noisy self-guided rewards.
Lastly, \textit{w/o vMF Modeling} replaces our distribution-based similarity with the naive cosine similarity against mean directions, degrading performance by $2.3\%$ ID and $2.0\%$ OOD, highlighting the criticality of global distribution modeling for robust confidence estimation.

\paragraph{Computational Efficiency Analysis.}
To evaluate the computational overhead introduced by GeoMin, we profile the training time allocation across different operational phases in Figure~\ref{fig:timing}. 
The breakdown reveals that trajectory rollouts (\textit{Generation}) and policy optimization (\textit{Update Actor}) dominate the training process, consuming $46.8\%$ and $36.8\%$ of the total time, respectively. 
In contrast, our vMF-based global distribution modeling and sample filtering require merely $5.7\%$ of the training duration. 
This marginal footprint highlights the efficiency of GeoMin, securing substantial performance gains with negligible algorithmic overhead.

\section{Related Work}
\noindent \textbf{Weakly-supervised RLVR} alleviates the dependency on costly ground-truth annotations through two primary paradigms: unsupervised and semi-supervised RLVR.
\textit{Unsupervised RLVR} derives rewards directly from internal model states. Specifically, confidence-based approaches \citep{agarwal2025unreasonable, zhao2025learning, van2025post, prabhudesai2025maximizing, li2025confidence} equate model certainty with correctness, while ensemble-based methods \citep{yuan2025wisdom, zuo2025ttrl, shafayat2025can, liu2025ettrl, fang2026serl, zhang2025co,  zhang2026right, zhang2026consistent} leverage rollout consensus as a proxy for truth. However, lacking factual grounding, these intrinsic signals often trigger reward hacking and late-stage training collapse \citep{he2026far}.
\textit{Semi-supervised RLVR} mitigates this by using a small labeled set to anchor unlabeled data filtering. Recent work, such as \citep{yang2025trapo}, aligns individual trajectory pass-rates with collective statistics. Yet, such macro heuristics overlook inherent individual fluctuations, causing a ``recall bottleneck'' that excludes valuable samples. Instead, we uncover the geometric properties of latent representations for high-precision and comprehensive sample mining.

\noindent\textbf{Semi-supervised learning (SSL)} mitigates label scarcity by leveraging unlabeled data. Most contemporary SSL approaches rest on two paradigms: consistency regularization \citep{miyato2018virtual,ke2019dual}, enforcing stable predictions under perturbations, and pseudo-labeling \citep{zhang2020detecting,cascante2021curriculum,han2026dicap}, treating confident predictions as targets. FixMatch \citep{sohn2020fixmatch} seamlessly integrates both via consistency across weakly and strongly augmented views, while subsequent advances \citep{zhang2021flexmatch,wang2022freematch,chen2023softmatch,li2024semireward,cheng2025cgmatch} adapt thresholds to balance label reliability and coverage. However, these classification-tailored methods rely heavily on surface-level output probabilities, which inherently fail to capture the complex, multi-step reasoning trajectories in RLVR. This limitation motivates us to look beyond token-level text outputs and leverage representation-level geometric alignment.

\section{Conclusion}
In this work, we propose GeoMin, a novel semi-supervised RLVR framework that fully unleashes the potential of unlabeled data through geometry-guided sample mining.
The core insight of GeoMin is that the underlying structure of latent representations exhibits a natural geometric resonance between labeled and unlabeled samples.
Specifically, this resonance allows vMF-modeled directional distributions to decode the structural discrepancy between correct and incorrect rollouts, establishing a robust prior for adaptive sample mining.
Strikingly, GeoMin surpasses the fully supervised baseline with merely $10\%$ annotation overhead, paving the way toward scalable, data-efficient RLVR.

\section*{Limitations}
Despite its effectiveness, GeoMin has several limitations. 
First, due to constrained computational resources, our empirical validation was primarily conducted on models scaling up to 8B parameters. Evaluating GeoMin on larger frontier models remains an essential next step. 
Second, GeoMin's training efficiency is tied to the model's initial representation quality. When applied to models with lower geometric discriminability at the outset, GeoMin requires a prolonged Stage 1 training phase to fully activate and align the latent geometry, introducing a trade-off between initial model readiness and training convergence speed.
Finally, our current evaluation focuses on textual and symbolic reasoning tasks within RLVR. Generalizing this geometric framework to more intricate latent landscapes, such as multimodal or cross-lingual reasoning, warrants further investigation.

\section*{Ethics Statement}
We have carefully reviewed the ethical implications of our work in accordance with the ACL Code of Ethics. As our proposed method is evaluated exclusively on standard, publicly available benchmarks for logical reasoning and utilizes no human-centric, private, or sensitive data, this work does not raise any ethical concerns, malicious content risks, or negative societal impacts.


\bibliography{custom}

@article{jaech2024openai,
  title={Openai o1 system card},
  author={Jaech, Aaron and Kalai, Adam and Lerer, Adam and Richardson, Adam and El-Kishky, Ahmed and Low, Aiden and Helyar, Alec and Madry, Aleksander and Beutel, Alex and Carney, Alex and others},
  journal={arXiv preprint arXiv:2412.16720},
  year={2024}
}

@article{guo2025deepseek,
  title={DeepSeek-R1 incentivizes reasoning in LLMs through reinforcement learning},
  author={Guo, Daya and Yang, Dejian and Zhang, Haowei and Song, Junxiao and Wang, Peiyi and Zhu, Qihao and Xu, Runxin and Zhang, Ruoyu and Ma, Shirong and Bi, Xiao and others},
  journal={Nature},
  volume={645},
  number={8081},
  pages={633--638},
  year={2025},
  publisher={Nature Publishing Group UK London}
}

@article{comanici2025gemini,
  title={Gemini 2.5: Pushing the frontier with advanced reasoning, multimodality, long context, and next generation agentic capabilities},
  author={Comanici, Gheorghe and Bieber, Eric and Schaekermann, Mike and Pasupat, Ice and Sachdeva, Noveen and Dhillon, Inderjit and Blistein, Marcel and Ram, Ori and Zhang, Dan and Rosen, Evan and others},
  journal={arXiv preprint arXiv:2507.06261},
  year={2025}
}

@article{yang2025qwen3,
  title={Qwen3 technical report},
  author={Yang, An and Li, Anfeng and Yang, Baosong and Zhang, Beichen and Hui, Binyuan and Zheng, Bo and Yu, Bowen and Gao, Chang and Huang, Chengen and Lv, Chenxu and others},
  journal={arXiv preprint arXiv:2505.09388},
  year={2025}
}

@article{wen2025reinforcement,
  title={Reinforcement learning with verifiable rewards implicitly incentivizes correct reasoning in base llms},
  author={Wen, Xumeng and Liu, Zihan and Zheng, Shun and Ye, Shengyu and Wu, Zhirong and Wang, Yang and Xu, Zhijian and Liang, Xiao and Li, Junjie and Miao, Ziming and others},
  journal={arXiv preprint arXiv:2506.14245},
  year={2025}
}

@article{zhang2025survey,
  title={A survey of reinforcement learning for large reasoning models},
  author={Zhang, Kaiyan and Zuo, Yuxin and He, Bingxiang and Sun, Youbang and Liu, Runze and Jiang, Che and Fan, Yuchen and Tian, Kai and Jia, Guoli and Li, Pengfei and others},
  journal={arXiv preprint arXiv:2509.08827},
  year={2025}
}

@article{su2025crossing,
  title={Crossing the reward bridge: Expanding rl with verifiable rewards across diverse domains},
  author={Su, Yi and Yu, Dian and Song, Linfeng and Li, Juntao and Mi, Haitao and Tu, Zhaopeng and Zhang, Min and Yu, Dong},
  journal={arXiv preprint arXiv:2503.23829},
  year={2025}
}

@article{he2026far,
  title={How Far Can Unsupervised RLVR Scale LLM Training?},
  author={He, Bingxiang and Zuo, Yuxin and Liu, Zeyuan and Zhao, Shangziqi and Fu, Zixuan and Yang, Junlin and Qian, Cheng and Zhang, Kaiyan and Fan, Yuchen and Cui, Ganqu and others},
  journal={arXiv preprint arXiv:2603.08660},
  year={2026}
}

@article{zhao2026absolute,
  title={Absolute zero: Reinforced self-play reasoning with zero data},
  author={Zhao, Andrew and Wu, Yiran and Wu, Tong and Xu, Quentin and Yue, Yang and Lin, Matthieu and Wang, Shenzhi and Wu, Qingyun and Zheng, Zilong and Huang, Gao},
  journal={Advances in Neural Information Processing Systems},
  volume={38},
  pages={105816--105879},
  year={2026}
}

@article{zuo2025ttrl,
  title={Ttrl: Test-time reinforcement learning},
  author={Zuo, Yuxin and Zhang, Kaiyan and Sheng, Li and Qu, Shang and Cui, Ganqu and Zhu, Xuekai and Li, Haozhan and Zhang, Yuchen and Long, Xinwei and Hua, Ermo and others},
  journal={arXiv preprint arXiv:2504.16084},
  year={2025}
}

@article{agarwal2025unreasonable,
  title={The unreasonable effectiveness of entropy minimization in llm reasoning},
  author={Agarwal, Shivam and Zhang, Zimin and Yuan, Lifan and Han, Jiawei and Peng, Hao},
  journal={arXiv preprint arXiv:2505.15134},
  year={2025}
}

@article{zhao2025learning,
  title={Learning to reason without external rewards},
  author={Zhao, Xuandong and Kang, Zhewei and Feng, Aosong and Levine, Sergey and Song, Dawn},
  journal={arXiv preprint arXiv:2505.19590},
  year={2025}
}

@article{zhang2025co,
  title={Co-rewarding: Stable Self-supervised RL for Eliciting Reasoning in Large Language Models},
  author={Zhang, Zizhuo and Zhu, Jianing and Ge, Xinmu and Zhao, Zihua and Zhou, Zhanke and Li, Xuan and Feng, Xiao and Yao, Jiangchao and Han, Bo},
  journal={arXiv preprint arXiv:2508.00410},
  year={2025}
}

@article{shafayat2025can,
  title={Can Large Reasoning Models Self-Train?},
  author={Shafayat, Sheikh and Tajwar, Fahim and Salakhutdinov, Ruslan and Schneider, Jeff and Zanette, Andrea},
  journal={arXiv preprint arXiv:2505.21444},
  year={2025}
}

@article{zhang2025no,
  title={No Free Lunch: Rethinking Internal Feedback for LLM Reasoning},
  author={Zhang, Yanzhi and Zhang, Zhaoxi and Guan, Haoxiang and Cheng, Yilin and Duan, Yitong and Wang, Chen and Wang, Yue and Zheng, Shuxin and He, Jiyan},
  journal={arXiv preprint arXiv:2506.17219},
  year={2025}
}

@article{yang2025trapo,
  title={TraPO: A Semi-Supervised Reinforcement Learning Framework for Boosting LLM Reasoning},
  author={Yang, Shenzhi and Zhu, Guangcheng and Zheng, Xing and MA, Yingfan and Chen, Zhongqi and Song, Bowen and Wang, Weiqiang and Zhao, Junbo and Chen, Gang and Wang, Haobo},
  journal={arXiv preprint arXiv:2512.13106},
  year={2025}
}

@article{xie2026controlled,
  title={Controlled llm training on spectral sphere},
  author={Xie, Tian and Luo, Haoming and Tang, Haoyu and Hu, Yiwen and Liu, Jason Klein and Ren, Qingnan and Wang, Yang and Zhao, Wayne Xin and Yan, Rui and Su, Bing and others},
  journal={arXiv preprint arXiv:2601.08393},
  year={2026}
}

@article{fu2026nemotron,
  title={Nemotron-flash: Towards latency-optimal hybrid small language models},
  author={Fu, Yonggan and Dong, Xin and Diao, Shizhe and Ye, Hanrong and Byeon, Wonmin and Karnati, Yashaswi and Liebenwein, Lucas and Khadkevich, Maksim and Keller, Alexander and Kautz, Jan and others},
  journal={Advances in Neural Information Processing Systems},
  volume={38},
  pages={161342--161366},
  year={2026}
}

@article{zhang2019root,
  title={Root mean square layer normalization},
  author={Zhang, Biao and Sennrich, Rico},
  journal={Advances in neural information processing systems},
  volume={32},
  year={2019}
}

@article{shao2024deepseekmath,
  title={Deepseekmath: Pushing the limits of mathematical reasoning in open language models},
  author={Shao, Zhihong and Wang, Peiyi and Zhu, Qihao and Xu, Runxin and Song, Junxiao and Bi, Xiao and Zhang, Haowei and Zhang, Mingchuan and Li, YK and Wu, Yang and others},
  journal={arXiv preprint arXiv:2402.03300},
  year={2024}
}

@article{du2024probabilistic,
  title={Probabilistic contrastive learning for long-tailed visual recognition},
  author={Du, Chaoqun and Wang, Yulin and Song, Shiji and Huang, Gao},
  journal={IEEE Transactions on Pattern Analysis and Machine Intelligence},
  year={2024},
  publisher={IEEE}
}

@article{sra2012short,
  title={A short note on parameter approximation for von Mises-Fisher distributions: and a fast implementation of I s (x)},
  author={Sra, Suvrit},
  journal={Computational Statistics},
  volume={27},
  pages={177--190},
  year={2012},
  publisher={Springer}
}

@article{he2025deepmath,
  title={Deepmath-103k: A large-scale, challenging, decontaminated, and verifiable mathematical dataset for advancing reasoning},
  author={He, Zhiwei and Liang, Tian and Xu, Jiahao and Liu, Qiuzhi and Chen, Xingyu and Wang, Yue and Song, Linfeng and Yu, Dian and Liang, Zhenwen and Wang, Wenxuan and others},
  journal={arXiv preprint arXiv:2504.11456},
  year={2025}
}

@inproceedings{sheng2025hybridflow,
  title={Hybridflow: A flexible and efficient rlhf framework},
  author={Sheng, Guangming and Zhang, Chi and Ye, Zilingfeng and Wu, Xibin and Zhang, Wang and Zhang, Ru and Peng, Yanghua and Lin, Haibin and Wu, Chuan},
  booktitle={Proceedings of the Twentieth European Conference on Computer Systems},
  pages={1279--1297},
  year={2025}
}

@article{yan2025learning,
  title={Learning to reason under off-policy guidance},
  author={Yan, Jianhao and Li, Yafu and Hu, Zican and Wang, Zhi and Cui, Ganqu and Qu, Xiaoye and Cheng, Yu and Zhang, Yue},
  journal={arXiv preprint arXiv:2504.14945},
  year={2025}
}

@article{li2024numinamath,
  title={Numinamath: The largest public dataset in ai4maths with 860k pairs of competition math problems and solutions},
  author={Li, Jia and Beeching, Edward and Tunstall, Lewis and Lipkin, Ben and Soletskyi, Roman and Huang, Shengyi and Rasul, Kashif and Yu, Longhui and Jiang, Albert Q and Shen, Ziju and others},
  journal={Hugging Face repository},
  volume={13},
  number={9},
  pages={9},
  year={2024}
}

@inproceedings{he2024olympiadbench,
  title={Olympiadbench: A challenging benchmark for promoting agi with olympiad-level bilingual multimodal scientific problems},
  author={He, Chaoqun and Luo, Renjie and Bai, Yuzhuo and Hu, Shengding and Thai, Zhen and Shen, Junhao and Hu, Jinyi and Han, Xu and Huang, Yujie and Zhang, Yuxiang and others},
  booktitle={Proceedings of the 62nd Annual Meeting of the Association for Computational Linguistics (Volume 1: Long Papers)},
  pages={3828--3850},
  year={2024}
}

@article{lewkowycz2022solving,
  title={Solving quantitative reasoning problems with language models},
  author={Lewkowycz, Aitor and Andreassen, Anders and Dohan, David and Dyer, Ethan and Michalewski, Henryk and Ramasesh, Vinay and Slone, Ambrose and Anil, Cem and Schlag, Imanol and Gutman-Solo, Theo and others},
  journal={Advances in neural information processing systems},
  volume={35},
  pages={3843--3857},
  year={2022}
}

@article{hendrycks2021measuring,
  title={Measuring mathematical problem solving with the math dataset},
  author={Hendrycks, Dan and Burns, Collin and Kadavath, Saurav and Arora, Akul and Basart, Steven and Tang, Eric and Song, Dawn and Steinhardt, Jacob},
  journal={arXiv preprint arXiv:2103.03874},
  year={2021}
}

@article{clark2018think,
  title={Think you have solved question answering? try arc, the ai2 reasoning challenge},
  author={Clark, Peter and Cowhey, Isaac and Etzioni, Oren and Khot, Tushar and Sabharwal, Ashish and Schoenick, Carissa and Tafjord, Oyvind},
  journal={arXiv preprint arXiv:1803.05457},
  year={2018}
}

@inproceedings{rein2024gpqa,
  title={Gpqa: A graduate-level google-proof q\&a benchmark},
  author={Rein, David and Hou, Betty Li and Stickland, Asa Cooper and Petty, Jackson and Pang, Richard Yuanzhe and Dirani, Julien and Michael, Julian and Bowman, Samuel R},
  booktitle={First Conference on Language Modeling},
  year={2024}
}

@article{wang2024mmlu,
  title={Mmlu-pro: A more robust and challenging multi-task language understanding benchmark},
  author={Wang, Yubo and Ma, Xueguang and Zhang, Ge and Ni, Yuansheng and Chandra, Abhranil and Guo, Shiguang and Ren, Weiming and Arulraj, Aaran and He, Xuan and Jiang, Ziyan and others},
  journal={Advances in Neural Information Processing Systems},
  volume={37},
  pages={95266--95290},
  year={2024}
}

@article{prabhudesai2025maximizing,
  title={Maximizing confidence alone improves reasoning},
  author={Prabhudesai, Mihir and Chen, Lili and Ippoliti, Alex and Fragkiadaki, Katerina and Liu, Hao and Pathak, Deepak},
  journal={arXiv preprint arXiv:2505.22660},
  year={2025}
}

@article{li2025confidence,
  title={Confidence is all you need: Few-shot rl fine-tuning of language models},
  author={Li, Pengyi and Skripkin, Matvey and Zubrey, Alexander and Kuznetsov, Andrey and Oseledets, Ivan},
  journal={arXiv preprint arXiv:2506.06395},
  year={2025}
}

@article{van2025post,
  title={Post-training large language models via reinforcement learning from self-feedback},
  author={van Niekerk, Carel and Vukovic, Renato and Ruppik, Benjamin Matthias and Lin, Hsien-chin and Ga{\v{s}}i{\'c}, Milica},
  journal={arXiv preprint arXiv:2507.21931},
  year={2025}
}

@article{liu2025ettrl,
  title={Ettrl: Balancing exploration and exploitation in llm test-time reinforcement learning via entropy mechanism},
  author={Liu, Jia and He, ChangYi and Lin, YingQiao and Yang, MingMin and Shen, FeiYang and Liu, ShaoGuo},
  journal={arXiv preprint arXiv:2508.11356},
  year={2025}
}

@article{fang2026serl,
  title={Serl: Self-play reinforcement learning for large language models with limited data},
  author={Fang, Wenkai and Liu, Shunyu and Zhou, Yang and Zhang, Kongcheng and Zheng, Tongya and Chen, Kaixuan and Song, Mingli and Tao, Dacheng},
  journal={Advances in Neural Information Processing Systems},
  volume={38},
  pages={103706--103738},
  year={2026}
}

@article{yuan2025wisdom,
  title={Wisdom of the Crowd: Reinforcement Learning from Coevolutionary Collective Feedback},
  author={Yuan, Wenzhen and Tang, Shengji and Lin, Weihao and Ruan, Jiacheng and Cui, Ganqu and Zhang, Bo and Chen, Tao and Liu, Ting and Fu, Yuzhuo and Ye, Peng and others},
  journal={arXiv preprint arXiv:2508.12338},
  year={2025}
}

@article{zhang2026right,
  title={Right question is already half the answer: Fully unsupervised llm reasoning incentivization},
  author={Zhang, Qingyang and Wu, Haitao and Zhang, Changqing and Zhao, Peilin and Bian, Yatao},
  journal={Advances in neural information processing systems},
  volume={38},
  pages={67345--67372},
  year={2026}
}

@article{zhang2026consistent,
  title={Consistent paths lead to truth: Self-rewarding reinforcement learning for llm reasoning},
  author={Zhang, Kongcheng and Yao, Qi and Liu, Shunyu and Wang, Yingjie and Lai, Baisheng and Ye, Jieping and Song, Mingli and Tao, Dacheng},
  journal={Advances in Neural Information Processing Systems},
  volume={38},
  pages={59849--59887},
  year={2026}
}

@article{miyato2018virtual,
  title={Virtual adversarial training: a regularization method for supervised and semi-supervised learning},
  author={Miyato, Takeru and Maeda, Shin-ichi and Koyama, Masanori and Ishii, Shin},
  journal={IEEE transactions on pattern analysis and machine intelligence},
  volume={41},
  number={8},
  pages={1979--1993},
  year={2018},
  publisher={IEEE}
}

@inproceedings{ke2019dual,
  title={Dual student: Breaking the limits of the teacher in semi-supervised learning},
  author={Ke, Zhanghan and Wang, Daoye and Yan, Qiong and Ren, Jimmy and Lau, Rynson WH},
  booktitle={Proceedings of the IEEE/CVF international conference on computer vision},
  pages={6728--6736},
  year={2019}
}

@article{zhang2020detecting,
  title={Detecting false data injection attacks in smart grids: A semi-supervised deep learning approach},
  author={Zhang, Ying and Wang, Jianhui and Chen, Bo},
  journal={IEEE Transactions on Smart Grid},
  volume={12},
  number={1},
  pages={623--634},
  year={2020},
  publisher={IEEE}
}

@inproceedings{cascante2021curriculum,
  title={Curriculum labeling: Revisiting pseudo-labeling for semi-supervised learning},
  author={Cascante-Bonilla, Paola and Tan, Fuwen and Qi, Yanjun and Ordonez, Vicente},
  booktitle={Proceedings of the AAAI conference on artificial intelligence},
  volume={35},
  number={8},
  pages={6912--6920},
  year={2021}
}

@inproceedings{han2026dicap,
  title={DiCaP: Distribution-Calibrated Pseudo-labeling for Semi-Supervised Multi-Label Learning},
  author={Han, Bo and Li, Zhuoming and Wang, Xiaoyu and Hou, Yaxin and Liu, Hui and Hou, Junhui and Jia, Yuheng},
  booktitle={Proceedings of the AAAI Conference on Artificial Intelligence},
  volume={40},
  number={26},
  pages={21540--21548},
  year={2026}
}

@article{sohn2020fixmatch,
  title={Fixmatch: Simplifying semi-supervised learning with consistency and confidence},
  author={Sohn, Kihyuk and Berthelot, David and Carlini, Nicholas and Zhang, Zizhao and Zhang, Han and Raffel, Colin A and Cubuk, Ekin Dogus and Kurakin, Alexey and Li, Chun-Liang},
  journal={Advances in neural information processing systems},
  volume={33},
  pages={596--608},
  year={2020}
}

@article{zhang2021flexmatch,
  title={Flexmatch: Boosting semi-supervised learning with curriculum pseudo labeling},
  author={Zhang, Bowen and Wang, Yidong and Hou, Wenxin and Wu, Hao and Wang, Jindong and Okumura, Manabu and Shinozaki, Takahiro},
  journal={Advances in neural information processing systems},
  volume={34},
  pages={18408--18419},
  year={2021}
}

@article{wang2022freematch,
  title={Freematch: Self-adaptive thresholding for semi-supervised learning},
  author={Wang, Yidong and Chen, Hao and Heng, Qiang and Hou, Wenxin and Fan, Yue and Wu, Zhen and Wang, Jindong and Savvides, Marios and Shinozaki, Takahiro and Raj, Bhiksha and others},
  journal={arXiv preprint arXiv:2205.07246},
  year={2022}
}

@article{chen2023softmatch,
  title={Softmatch: Addressing the quantity-quality trade-off in semi-supervised learning},
  author={Chen, Hao and Tao, Ran and Fan, Yue and Wang, Yidong and Wang, Jindong and Schiele, Bernt and Xie, Xing and Raj, Bhiksha and Savvides, Marios},
  journal={arXiv preprint arXiv:2301.10921},
  year={2023}
}

@inproceedings{li2024semireward,
  title={Semireward: A general reward model for semi-supervised learning},
  author={Li, Siyuan and Jin, Weiyang and Wang, Zedong and Wu, Fang and Liu, Zicheng and Tan, Cheng and Li, Stan Z},
  booktitle={International Conference on Learning Representations},
  volume={2024},
  pages={25501--25525},
  year={2024}
}

@inproceedings{cheng2025cgmatch,
  title={Cgmatch: A different perspective of semi-supervised learning},
  author={Cheng, Bo and Lu, Jueqing and Tian, Yuan and Zhao, Haifeng and Chang, Yi and Du, Lan},
  booktitle={Proceedings of the Computer Vision and Pattern Recognition Conference},
  pages={15381--15391},
  year={2025}
}

@article{grattafiori2024llama,
  title={The llama 3 herd of models},
  author={Grattafiori, Aaron and Dubey, Abhimanyu and Jauhri, Abhinav and Pandey, Abhinav and Kadian, Abhishek and Al-Dahle, Ahmad and Letman, Aiesha and Mathur, Akhil and Schelten, Alan and Vaughan, Alex and others},
  journal={arXiv preprint arXiv:2407.21783},
  year={2024}
}

@article{yang2024qwen2,
  title={Qwen2. 5-math technical report: Toward mathematical expert model via self-improvement},
  author={Yang, An and Zhang, Beichen and Hui, Binyuan and Gao, Bofei and Yu, Bowen and Li, Chengpeng and Liu, Dayiheng and Tu, Jianhong and Zhou, Jingren and Lin, Junyang and others},
  journal={arXiv preprint arXiv:2409.12122},
  year={2024}
}

\newpage
\appendix

\section{Theoretical Analysis}
\subsection{Proof of Proposition 1}\label{sec:proof_prop_1}
To capture the directional alignment and distributional affinity between a query feature and different rollout classes over the continuous latent space, we analyze the class-conditional distribution on the unit hypersphere. 
Let $\bm{z} \in \mathbb{S}^{d-1}$ denote a query feature vector. For a given class $c \in \{0, 1\}$ (representing incorrect and correct rollouts, respectively), instead of relying on a sparse set of discrete samples, we define the log-expected kernel density $\rho(\bm{z}, c)$ as the continuous expectation of a directional von Mises-Fisher (vMF) kernel over the entire class-conditional distribution $p(\bm{z}'|c)$. Specifically, using a hyperspherical coordinate kernel, $\rho(\bm{z}, c)$ is formulated as:
\begin{equation}
\label{eq:rho_expectation_def}
\scalebox{0.98}{$
\rho(\bm{z}, c) = \log \left( \mathbb{E}_{\bm{z}_c \sim p(\bm{z}'|c)} \left[\exp({\bm{z}^\top \bm{z}_c)}\right] \right)
$}
\end{equation}
where $\bm{z}_c$ is a random variable distributed according to the class-conditional density $p(\bm{z}'|c)$. This formulation implicitly serves as a non-parametric wrapper that smooths and represents the overall geometric alignment between the query $\bm{z}$ and the cluster of class $c$.

Following directional statistics, data constraints on a unit hypersphere $\mathbb{S}^{d-1}$ are naturally modeled by the vMF distribution, which serves as the spherical analogue of the isotropic Gaussian distribution. Assuming that the underlying class-conditional density follows a $d$-variate vMF distribution, i.e., $p(\bm{z}'|c) = \text{vMF}(\bm{\mu}_c, \kappa_c)$, we restate Proposition 1 and provide its analytical derivation below.

\noindent \textbf{Proposition 1.} \textit{Let $\bm{z} \in \mathbb{S}^{d-1}$ be the query feature, and let the class-conditional features $\bm{z}_c$ follow a von Mises-Fisher distribution $\text{vMF}(\bm{\mu}_c, \kappa_c)$.
The log-expected kernel density, defined as $\rho(\bm{z}, c) = \log(\mathbb{E}_{\bm{z}_c} [\exp(\bm{z}^\top \bm{z}_c)])$, admits the following closed-form expression:
\begin{equation}
\rho(\bm{z}, c) = \log C_d(\kappa_c) - \log C_d(\kappa_c'),
\end{equation}
where $\kappa_c' = \| \kappa_c \bm{\mu}_c + \bm{z} \|_2$. }

\begin{proof}
The probability density function of a $d$-variate von Mises-Fisher distribution $\text{vMF}(\bm{\mu}_c, \kappa_c)$ for a unit vector $\bm{z}_c \in \mathbb{S}^{d-1}$ is defined as:
\begin{equation}
f(\bm{z}_c; \bm{\mu}_c, \kappa_c) = C_d(\kappa_c) \exp({\kappa_c \bm{\mu}_c^\top \bm{z}_c}).
\end{equation}

We begin by expanding the expectation inside the logarithm of $\rho(\bm{z}, c)$ as an integral over the hypersphere $\mathbb{S}^{d-1}$ with respect to the normalized surface measure $d\omega(\bm{z}_c)$:
\begin{align}
\mathbb{E}&_{\bm{z}_c \sim \text{vMF}(\bm{\mu}_c, \kappa_c)} \left[ e^{\bm{z}^\top \bm{z}_c} \right] \nonumber \\
&= \int_{\mathbb{S}^{d-1}} e^{\bm{z}^\top \bm{z}_c} \cdot C_d(\kappa_c) e^{\kappa_c \bm{\mu}_c^\top \bm{z}_c} \, d\omega(\bm{z}_c) \nonumber \\
&= C_d(\kappa_c) \int_{\mathbb{S}^{d-1}} e^{\left(\kappa_c \bm{\mu}_c + \bm{z}\right)^\top \bm{z}_c} \, d\omega(\bm{z}_c).
\end{align}

To solve this integral, we introduce a composite vector $\bm{v} = \kappa_c \bm{\mu}_c + \bm{z}$. By factoring out its $\ell_2$-norm $\kappa_c' = \|\bm{v}\|_2$, we can rewrite the exponent, effectively combining the two directional components into a single vMF-like kernel:
\begin{equation}
\left(\kappa_c \bm{\mu}_c + \bm{z}\right)^\top \bm{z}_c = \kappa_c' \left( \frac{\bm{v}}{\kappa_c'} \right)^\top \bm{z}_c = \kappa_c' \bm{\mu}_c'^\top \bm{z}_c,
\end{equation}
where $\bm{\mu}_c' = \bm{v} / \kappa_c'$ is a valid unit vector since $\|\bm{\mu}_c'\|_2 = 1$.
Since $\bm{\mu}_c'$ represents a constant mean direction on the hypersphere, the total integral of $\exp(\kappa_c' \bm{\mu}_c'^\top \bm{z}_c)$ over $\mathbb{S}^{d-1}$ is simply the reciprocal of the normalization constant for a vMF distribution parameterized by concentration $\kappa_c'$. Due to the rotational symmetry of the hypersphere, this normalization factor $C_d(\kappa_c')$ depends strictly on the concentration parameter $\kappa_c'$ and dimension $d$, independent of the specific mean direction $\bm{\mu}_c'$:
\begin{equation}
\int_{\mathbb{S}^{d-1}} e^{\kappa_c' \bm{\mu}_c'^\top \bm{z}_c} \, d\omega(\bm{z}_c) = \frac{1}{C_d(\kappa_c')}.
\end{equation}

Substituting this back into our integral equation, the continuous expectation simplifies to the ratio of the two normalization factors:
\begin{equation}
\mathbb{E}_{\bm{z}_c \sim \text{vMF}(\bm{\mu}_c, \kappa_c)} \left[ e^{\bm{z}^\top \bm{z}_c} \right] = \frac{C_d(\kappa_c)}{C_d(\kappa_c')}.
\end{equation}

Finally, applying the natural logarithm to both sides yields the final analytical expression:
\begin{align}
\rho(\bm{z}, c) &= \log \left( \frac{C_d(\kappa_c)}{C_d(\kappa_c')} \right) \nonumber \\
&= \log C_d(\kappa_c) - \log C_d(\kappa_c').
\end{align}
This completes the proof.
\end{proof}

\section{Additional Experimental Setups}
\label{app:exp_setup}
\subsection{More Implementation Details.}
The maximum input prompt length is constrained to 1,024 tokens, and the generated response length is capped at 4,096 tokens for long-chain reasoning. All rollout prompts are prepended with the standard system prompt: ``\texttt{Let's think step by step and output the final answer within \textbackslash\textbackslash boxed\{\}.}''. To encourage exploration, we omit any form of KL regularization or entropy loss. By default, all models and baselines are trained for a total of 400 global steps; for GeoMin, Stage 1 spans the first 100 steps for representation anchoring, while Stage 2 occupies the remaining 300 steps for adaptive sample mining.

We employ \texttt{Math-Verify}\footnote{\url{https://github.com/huggingface/Math-Verify}} as the sole outcome reward function, deliberately avoiding auxiliary rewards for formatting constraints or response lengths to prevent reward hacking. Finally, rollout generation is accelerated via the \texttt{vLLM} framework\footnote{\url{https://github.com/vllm-project/vllm}} with 8 rollouts per prompt, using a fixed temperature of 1.0 and a top-$p$ value of 1.0. Model checkpoints are archived every 50 training steps for downstream validation.

\subsection{Computation of Normalization Factor}
\label{sec:norm_compute}
During implementation, evaluating the normalization factor $C_d(\kappa)$ of the vMF distribution requires integrating the probability density function to unity over the high-dimensional hypersphere, formulated as:
\begin{equation}
\label{eq:norm_factor_1}
C_d(\kappa) = \frac{\kappa^{d/2-1}}{(2\pi)^{d/2} I_{d/2-1}(\kappa)}.
\end{equation}
A key technical challenge in our setup is the robust computation of the modified Bessel function of the first kind, $I_{d/2-1}(\kappa)$, which admits the following power series expansion:
\begin{equation}
\label{eq:bessel}
I_{d/2-1}(\kappa) = \sum_{k=0}^{\infty} \frac{1}{k! \Gamma(d/2+k)} \left( \frac{\kappa}{2} \right)^{2k + d/2 - 1}.
\end{equation}
Directly evaluating this high-order expansion is computationally prohibitive, and standard forward recurrence triggers severe numerical tracking errors or underflow when $\kappa$ is small. 
To ensure training stability, we adopt the Miller recurrence algorithm following \citet{du2024probabilistic} to execute the backward recurrence relation:
\begin{equation}
\hat{I}_{\nu-1}(\kappa) = \frac{2\nu}{\kappa} \hat{I}_{\nu}(\kappa) + \hat{I}_{\nu+1}(\kappa).
\end{equation}
Specifically, we set a sufficiently large truncation limit $L_{\max}$ (where $L_{\max} \gg d/2-1$) and initialize the trial values to $\hat{I}_{L_{\max}}(\kappa) = 1$ and $\hat{I}_{L_{\max}+1}(\kappa) = 0$. We then iteratively compute $\hat{I}_{\nu}(\kappa)$ downward from $\nu = L_{\max}-1$ down to $0$. The exact value of the target order $I_{d/2-1}(\kappa)$ is subsequently recovered via:
\begin{equation}
I_{d/2-1}(\kappa) = \frac{I_0(\kappa)}{\hat{I}_0(\kappa)} \hat{I}_{d/2-1}(\kappa),
\end{equation}
where the baseline $I_0(\kappa)$ is efficiently and stably computed using the native \texttt{torch.special.i0} kernel in PyTorch.

\section{Additional Experimental Results}
\label{app:exp_results}

\begin{table*}[!t]
\centering
\caption{\small In-domain (ID) and out-of-domain (OOD) performance using Deepseek-R1-Distill-Llama-8B. Methods are evaluated
under semi-supervised (10\% labeled data) settings. \textbf{Bold} denotes the best results.}
\label{tab:results_llama}
\setlength{\tabcolsep}{3pt}  
\renewcommand{\arraystretch}{1.1} 
\resizebox{\textwidth}{!}{%
\begin{tabular}{lcccccc|cccc}
\toprule
\multirow{2}{*}{\textbf{Methods}} & \multicolumn{6}{c}{\textbf{In-Domain Performance}} & \multicolumn{4}{c}{\textbf{Out-of-Domain Performance}} \\
\cmidrule(lr){2-7} \cmidrule(lr){8-11}
 & \textbf{AIME 24/25} & \textbf{AMC} & \textbf{MATH-500} & \textbf{Minerva} & \textbf{Olympiad} & \textbf{Avg.} & \textbf{ARC-c} & \textbf{GPQA}$^{*}$ & \textbf{MMLU-Pro} & \textbf{Avg.} \\
\midrule
TTRL & 33.3/20.0 & \textbf{73.5} & 80.4 & 30.9 & 47.4 & 47.6 & 34.8 & \textbf{37.6} & 51.4 & 41.3 \\
Tok-entropy & 6.7/6.7 & 36.1 & 59.4 & 19.1 & 29.3 & 26.2 & 48.9 & 11.7 & 38.0 & 32.9 \\
Seq-entropy & 3.3/10.0 & 42.2 & 56.2 & 22.4 & 29.4 & 27.3 & 40.3 & 15.2 & 38.2 & 31.2 \\
Self-certainty & 16.7/16.7 & 51.8 & 73.0 & 26.1 & 43.8 & 38.0 & 25.3 & 16.2 & 44.1 & 28.5 \\
Co-rewarding & 30.0/20.0 & 69.9 & 79.6 & 28.7 & 48.1 & 46.1 & 31.3 & 29.9 & 51.2 & 37.5 \\
TraPO & 33.3/\textbf{23.3} & 67.5 & 81.0 & 32.0 & 52.7 & 48.3 & 34.6 & 34.5 & 51.1 & 40.1 \\
\textbf{GeoMin (ours)} & \textbf{40.0}/\textbf{23.3} & \textbf{73.5} & \textbf{83.6} & \textbf{32.7} & \textbf{55.3} & \textbf{51.4} & \textbf{65.1} & 36.0 & \textbf{51.6} & \textbf{50.9} \\
\bottomrule
% \vspace{-35pt}
\end{tabular}%
}
\end{table*}

\begin{table*}[!t]
\centering
\caption{\small In-domain (ID) and out-of-domain (OOD) performance using Deepseek-R1-Distill-Qwen-1.5B. Methods are evaluated
under semi-supervised (10\% labeled data) settings. \textbf{Bold} denotes the best results.}
\label{tab:results_qwen}
\setlength{\tabcolsep}{3pt}  
\renewcommand{\arraystretch}{1.1} 
\resizebox{\textwidth}{!}{%
\begin{tabular}{lcccccc|cccc}
\toprule
\multirow{2}{*}{\textbf{Methods}} & \multicolumn{6}{c}{\textbf{In-Domain Performance}} & \multicolumn{4}{c}{\textbf{Out-of-Domain Performance}} \\
\cmidrule(lr){2-7} \cmidrule(lr){8-11}
 & \textbf{AIME 24/25} & \textbf{AMC} & \textbf{MATH-500} & \textbf{Minerva} & \textbf{Olympiad} & \textbf{Avg.} & \textbf{ARC-c} & \textbf{GPQA}$^{*}$ & \textbf{MMLU-Pro} & \textbf{Avg.} \\
\midrule
TTRL & \textbf{23.3}/20.0 & 54.2 & 77.6 & 25.7 & 40.9 & 40.8 & 32.3 & 20.3 & \textbf{30.7} & 27.8 \\
Tok-entropy & 6.7/10.0 & 32.5 & 70.2 & 27.2 & 33.6 & 30.0 & 32.8 & 18.3 & 30.3 & 27.1 \\
Seq-entropy & 13.3/10.0 & 42.2 & 72.4 & 27.6 & 34.8 & 33.4 & 31.5 & 20.3 & 30.4 & 27.4 \\
Self-certainty & 13.3/13.3 & 56.6 & 75.6 & 27.6 & 40.3 & 37.8 & 31.4 & 19.3 & 30.6 & 27.1 \\
Co-rewarding & 13.3/20.0 & 57.8 & 77.6 & 24.6 & 39.5 & 38.8 & 32.3 & 17.3 & 30.5 & 26.7 \\ 
TraPO & 16.7/13.3 & 60.2 & 78.2 & 24.3 & 41.9 & 39.1 & 32.6 & 16.8 & 30.5 & 26.6 \\
\textbf{GeoMin (ours)} & 20.0/\textbf{23.3} & \textbf{67.5} & \textbf{79.6} & \textbf{32.0} & \textbf{49.6} & \textbf{45.3} & 32.0 & \textbf{22.3} & \textbf{30.7} & \textbf{28.3} \\
\bottomrule
% \vspace{-35pt}
\end{tabular}%
}
\end{table*}

\subsection{Extend GeoMin to More Models}
\label{app:more_model}
To thoroughly evaluate the generalizability and robustness of GeoMin, we extend our framework to a broader range of base models across distinct architectural families and parameter scales. Specifically, we conduct extensive semi-supervised RLVR experiments utilizing DeepSeek-R1-Distill-Qwen-1.5B and DeepSeek-R1-Distill-Llama-8B, which represent distinct architectural lineages (Qwen~\citep{yang2024qwen2} vs. Llama~\citep{grattafiori2024llama}) and model sizes (1.5B vs. 8B), respectively.

First, the experimental results for DeepSeek-R1-Distill-Llama-8B are detailed in Table~\ref{tab:results_llama}. GeoMin demonstrates a substantial performance leap, outperforming the strongest baseline by $3.1\%$ in average In-Distribution (ID) accuracy and $9.6\%$ in average Out-of-Distribution (OOD) accuracy. Strikingly, we observe a pronounced performance variance among all compared methods on the OOD benchmarks for this 8B model, with accuracies spanning a wide range from $28.5\%$ to $50.9\%$. This high sensitivity indicates that the larger architecture is highly susceptible to the quality of sample mining during reinforcement learning. 

For the smaller model, DeepSeek-R1-Distill-Qwen-1.5B, the comprehensive evaluations are summarized in Table~\ref{tab:results_qwen}. GeoMin achieves consistent gains, surpassing the best-performing baseline by $4.5\%$ on ID tasks and $0.5\%$ on OOD benchmarks. In sharp contrast to the 8B model, the OOD performance profiles of all methods on this 1.5B distilled model are tightly clustered within a narrow margin of $26.6\%$ to $28.3\%$, suggesting an inherent performance bottleneck or capacity saturation.

Notably, despite these radically different empirical variance profiles and model behaviors across scales, GeoMin consistently secures SOTA results. This consistent superiority strongly demonstrates that GeoMin is an architecture-agnostic, scalable, and robust paradigm for unlocking the latent potential of unlabeled data.

\begin{table*}[!t]
\centering
\caption{\small Ablation study on token representation pooling strategies.}
\label{tab:ablation_annotation}
\setlength{\tabcolsep}{5pt}  
\renewcommand{\arraystretch}{1.1} 
\resizebox{\textwidth}{!}{%
\begin{tabular}{lcccccc|cccc}
\toprule
\multirow{2}{*}{\textbf{Methods}} & \multicolumn{6}{c}{\textbf{In-Domain Performance}} & \multicolumn{4}{c}{\textbf{Out-of-Domain Performance}} \\
\cmidrule(lr){2-7} \cmidrule(lr){8-11}
 & \textbf{AIME 24/25} & \textbf{AMC} & \textbf{MATH-500} & \textbf{Minerva} & \textbf{Olympiad} & \textbf{Avg.} & \textbf{ARC-c} & \textbf{GPQA}$^{*}$ & \textbf{MMLU-Pro} & \textbf{Avg.} \\
\midrule
Average & 18.0/15.5 & 57.0 & 83.3 & 43.5 & 49.9 & 44.5 & 92.6 & 45.2 & 64.7 & 67.5 \\
\textbf{Last Token} & \textbf{24.5}/\textbf{19.2} & \textbf{58.5} & \textbf{86.7} & \textbf{45.5} & \textbf{52.7} & \textbf{47.9} & \textbf{93.7} & \textbf{48.2} & \textbf{66.5} & \textbf{69.5} \\
\bottomrule
\end{tabular}%
}
\end{table*}

\begin{figure*}[ht]
	\centering
    \subfigure[Ablation on $\alpha$\label{fig:alpha_ablation}]{
		\includegraphics[width=0.31\textwidth]{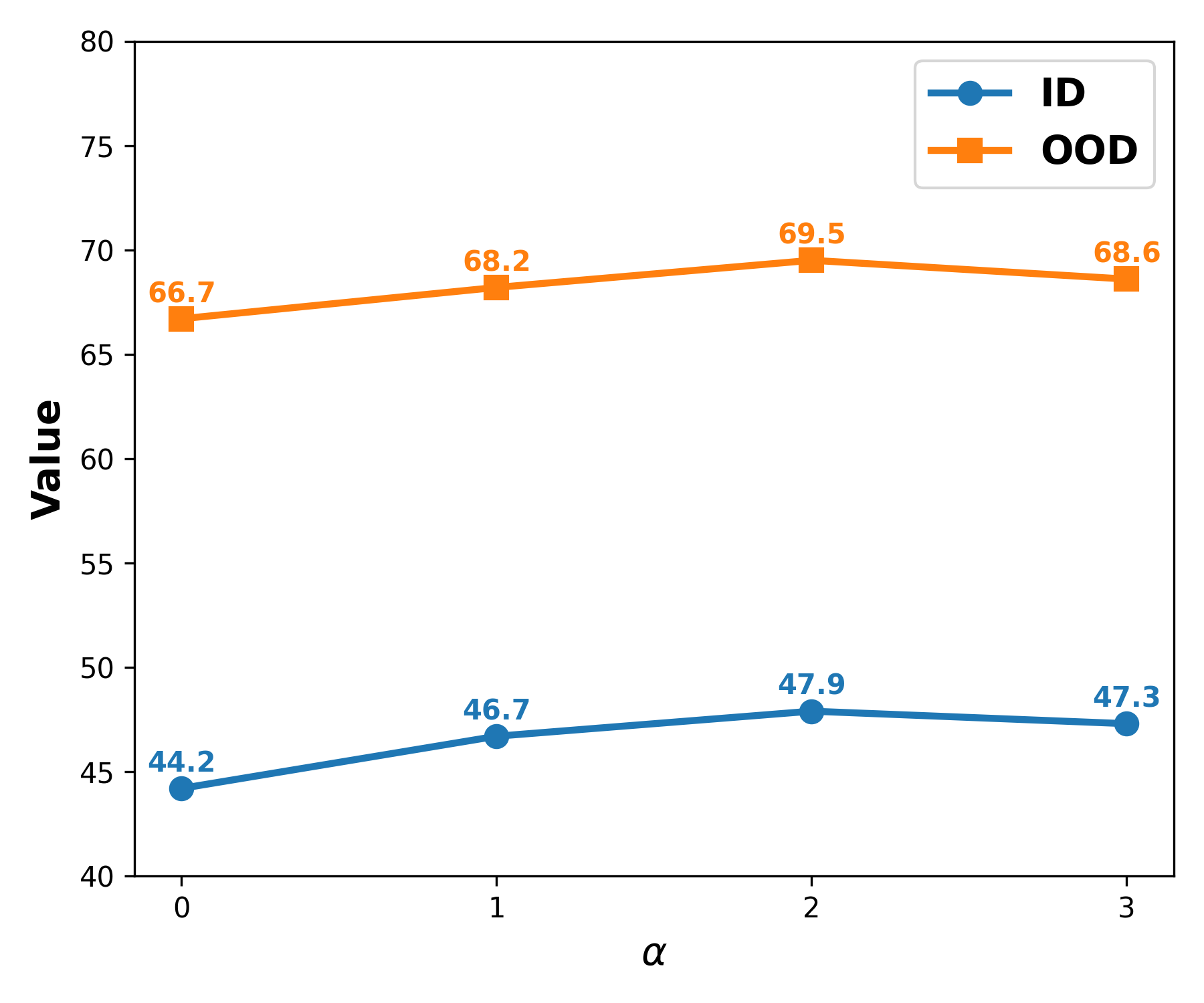}
	}
    \subfigure[Ablation on $K$\label{fig:topk_ablation}]{
		\includegraphics[width=0.31\textwidth]{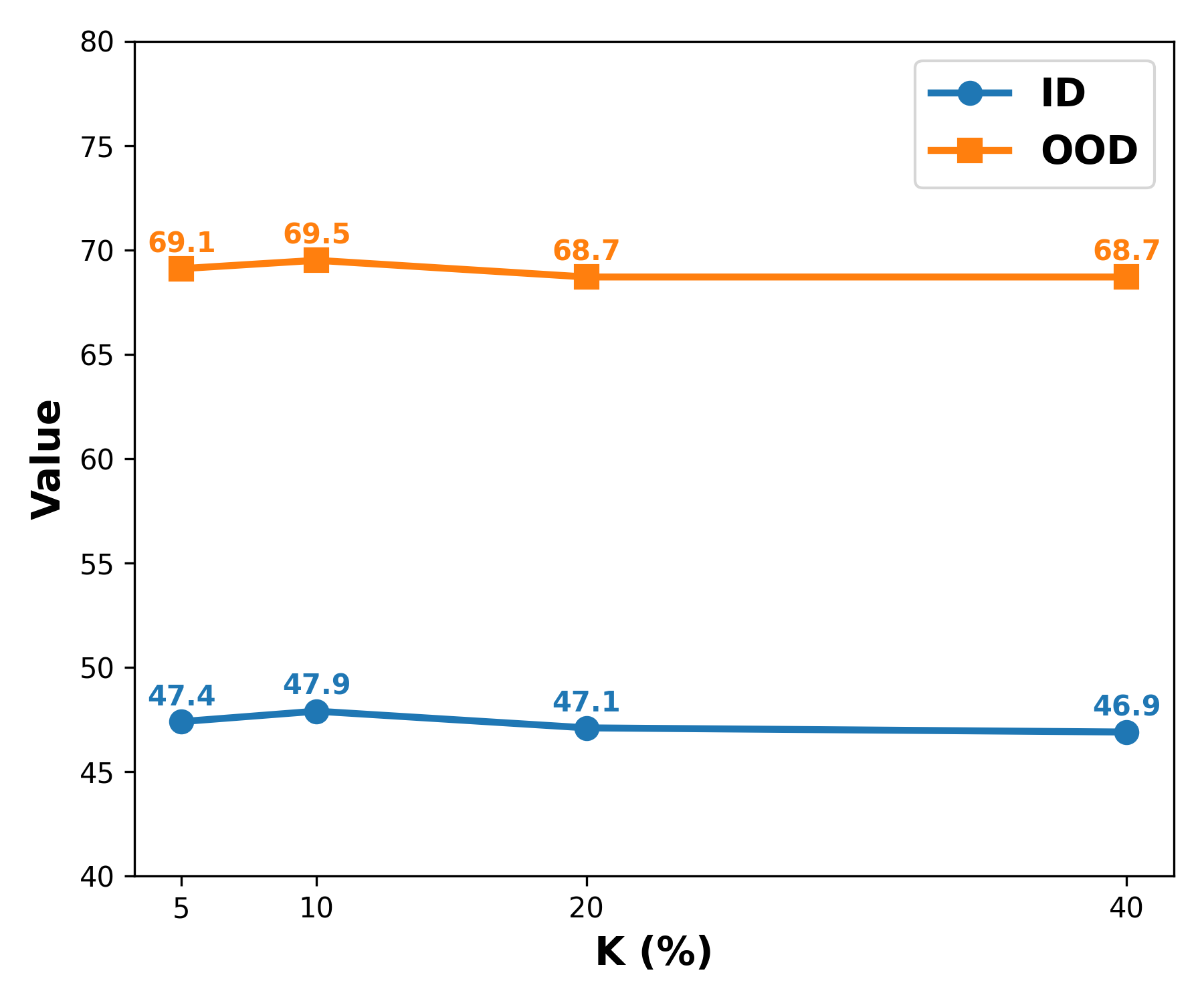}
	}
    \subfigure[Ablation on $\tau$\label{fig:tau_ablation}]{
		\includegraphics[width=0.31\textwidth]{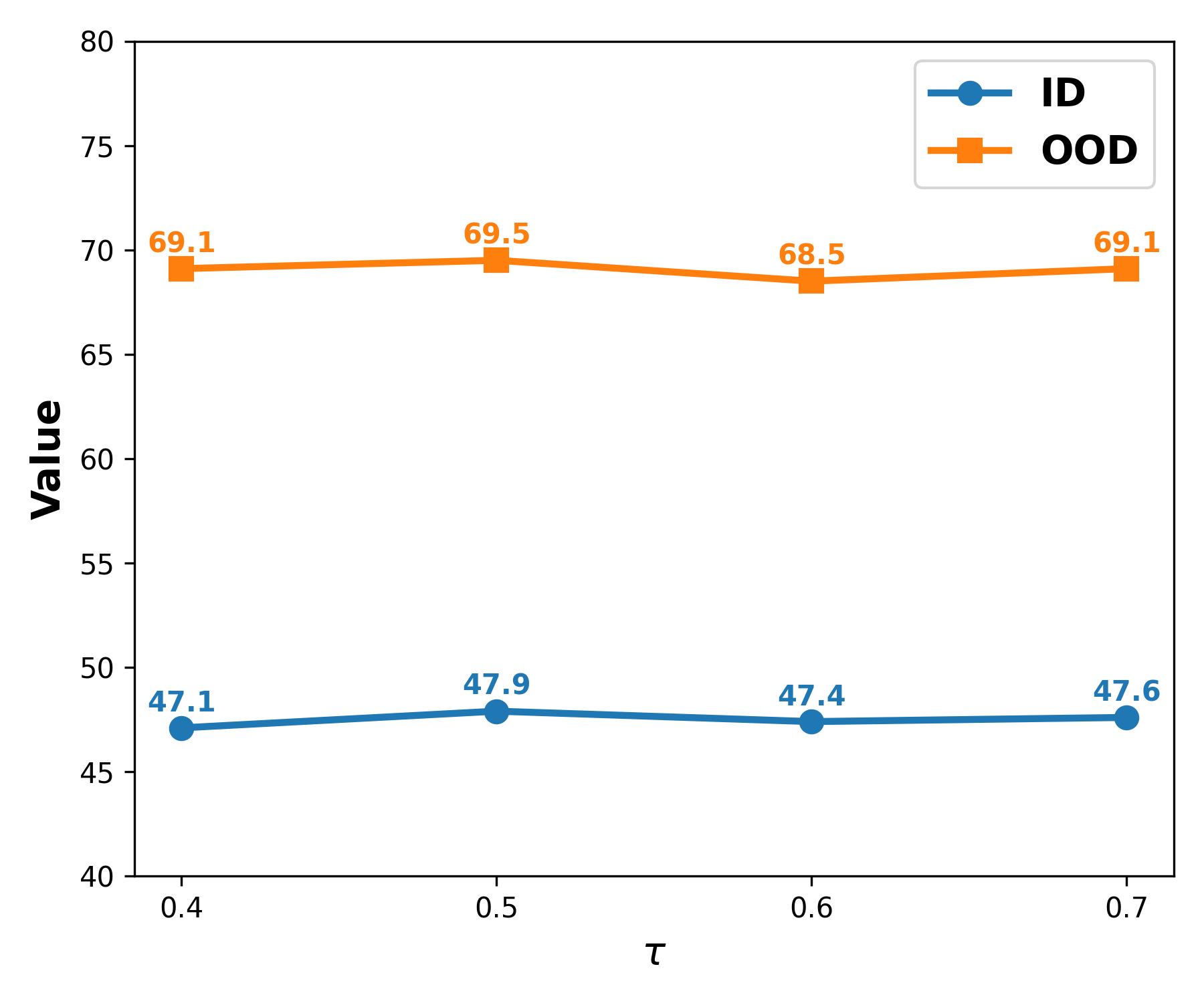}
	}
   \caption{\small Hyperparameter sensitivity analysis on advantage reweighting factor $\alpha$, top-$K$ layers, and GMM filtering threshold $\tau$.}
   \label{fig:hyper_sensitivity}
\end{figure*}

\subsection{Hyperparameter Sensitivity Analysis}
\label{sec:hyper_sensitivity}
A key desideratum for a robust semi-supervised framework is its insensitivity to hyperparameter tuning, ensuring stable deployment across diverse scenarios. To demonstrate this structural advantage, we evaluate the stability of GeoMin by conducting extensive hyperparameter sensitivity investigations on three core parameters: the advantage reweighting factor $\alpha$, the number of top-$K$ layers, and the GMM filtering threshold $\tau$. Crucially, as detailed below, GeoMin exhibits remarkable resilience across all tested parameter ranges, with performance remaining consistently high and tightly bounded around the optimal default configurations. All empirical evaluations are conducted using the default backbone model, with the comprehensive results summarized in Figure~\ref{fig:hyper_sensitivity}.

\paragraph{Advantage Reweighting Factor $\alpha$.}
The parameter $\alpha$ controls the geometric focus on informative boundary samples during the supervised anchoring phase. As empirically illustrated in Figure~\ref{fig:hyper_sensitivity}, varying $\alpha \in \{0, 1, 2, 3\}$ showcases a clear bell-shaped performance curve, where the model achieves its peak performance on both ID and OOD benchmarks at the default choice $\alpha = 2$. 
From a mechanistic perspective, setting $\alpha$ too small ($\alpha \le 1$) limits performance because the model pays insufficient attention to critical boundary samples, leading to inadequate representation discriminability in the latent space. Conversely, an excessively large $\alpha$ ($\alpha = 3$) overlooks confident samples, forcing the policy to heavily overfit to a minority of ambiguous samples while neglecting foundational learning signals from the broader distribution. Thus, $\alpha = 2$ strikes an optimal balance.

\paragraph{Number of Top-$K$ Layers.}
The hyperparameter $K$ determines the receptive field of layers selected for confidence score computation. When sweeping $K$ across $\{5\%, 10\%, 20\%, 40\%\}$, GeoMin yields tightly bounded accuracy scores, with the optimal performance materialized at our default configuration ($K=10\%$). This peak highlights a distinct trade-off in representation capacity: when $K$ is too small ($K=5\%$), the localized layer subset fails to encapsulate sufficient hierarchical context, causing an information scarcity bottleneck. On the other hand, inflating $K$ to larger scales ($K \ge 20\%$) introduces non-discriminative or noisy deep layers into the evaluation pool, effectively diluting the overall confidence calculation.

\paragraph{GMM Filtering Threshold $\tau$.}
The threshold $\tau$ governs the filtration criteria during unlabeled sample mining. When adjusting $\tau$ from $0.4$ to $0.7$, the resulting ID and OOD metrics shift minimally and remain highly stable, achieving a mild optimum at $\tau = 0.5$. Theoretically, an excessively low threshold relaxes the acceptance boundaries, riskily incorporating noisy pseudo-labels that poison policy training. Conversely, a rigid threshold proves overly restrictive, filtering out valuable sparse samples and resulting in data under-utilization. Crucially, the remarkably minor performance fluctuations across the entire spectrum of $\tau$ strongly underscore that GeoMin is inherently robust to threshold perturbations. This exceptional resilience empirically validates that our GMM-driven fitting operates as a highly adaptive and self-regulating selection mechanism, capturing true data distributions rather than relying on brittle manual heuristics.

\subsection{Effect of Embedding Pooling Strategies}
\label{sec:ablation_pooling}
To investigate the optimal strategy for extracting sequence-level representations within GeoMin, we conduct an ablation study comparing two distinct pooling mechanisms: (1) Average Pooling (denoted as \textit{Average}), which computes the arithmetic mean of embeddings across all tokens within the entire rollout sequence, and (2) Last Token Embedding (denoted as \textit{Last Token}), which leverages the representations of the final token as the sequence-level summary. As illustrated in Table~\ref{tab:ablation_annotation}, the \textit{Last Token} strategy consistently outperforms the \textit{Average} baseline. Specifically, the Last Token configuration secures a $+3.4\%$ absolute gain in ID average accuracy and a solid $+2.0\%$ improvement in OOD average accuracy. 

From a mechanistic perspective, this disparity stems from the causal self-attention mechanism inherent in modern autoregressive language models. In long-form multi-step reasoning rollouts, the final token naturally acts as an informational bottleneck that encapsulates and compresses the contextual dependencies, logical chains, and semantic states of the entire sequence. Instead, applying a naive average pooling across the entire sequence heavily dilutes these crucial task-specific representations by treating non-discriminative transition tokens and static prompt fillers with equal weight. This indiscriminate averaging introduces substantial representation redundancy and exacerbates the anisotropy problem, thereby contaminating the geometric alignment targeted by GeoMin with low-level token noise. The superior resilience and SOTA performance materialized by the Last Token strategy firmly validate its efficacy in capturing the intrinsic reasoning resonance of generative trajectories, proving that isolating the final state is vital for high-fidelity semantic clustering.

\begin{table}[h]
\centering
\caption{\small Quantitative comparison of total training wall-clock time and speedup ratio to reach the best checkpoint.}
\label{tab:efficiency_comparison}
\setlength{\tabcolsep}{8pt}
\renewcommand{\arraystretch}{1.1}
\small
\begin{tabular}{lcc}
\toprule
\textbf{Method} & \textbf{Total Training Time} & \textbf{Speedup} \\
\midrule
TraPO & 28h 44m 58s & 1.00$\times$ \\
\textbf{GeoMin (Ours)} & \textbf{13h 24m 36s} & \textbf{2.14$\times$} \\
\bottomrule
\end{tabular}
\end{table}

\subsection{Computational Overhead Comparison}
\label{sec:efficiency}
To evaluate the practical viability and scalability of GeoMin, we record and compare the total computational wall-clock time required to reach the best checkpoint between our framework and the competitive baseline, TraPO. Specifically, GeoMin completes its training and reaches convergence in \textbf{13h 24m 36s}. In contrast, TraPO demands a significantly heavier footprint, requiring \textbf{28h 44m 58s} to achieve the same milestone. This sharp contrast reveals that GeoMin achieves a substantial reduction in total training overhead, yielding a prominent \textbf{2.14$\times$ speedup} over TraPO.

From an architectural standpoint, this discrepancy in computational overhead is intrinsically rooted in TraPO's training paradigm, which requires an extensive multi-epoch warm-up phase to establish historical pass-rate trajectories. During this protracted stage, a vast amount of unlabeled data is utilized exclusively for rolling out sequences without contributing to any gradient updates, thereby incurring purely wasteful computational overhead. This decoupled mechanism forces the system to squander valuable hardware resources on heavy inference workloads that are entirely isolated from optimization. Consequently, such an inefficient design imposes a severe scalability bottleneck, making TraPO prohibitively expensive to deploy on large-scale unlabeled datasets.

Crucially, this architectural flaw not only drains computing budgets but also inherently restricts TraPO's performance upper bound. Because the optimization during the lengthy warm-up phase is strictly confined to a limited pool of labeled data, the model rapidly exhausts its learning capacity and undergoes severe optimization saturation. In information-theoretic terms, the available optimization entropy is prematurely consumed by over-indexing on labeled samples. By the time the unlabeled data is finally integrated into the gradient stream in later stages, the model's policy has already drifted into a rigid regime with diminished optimization headroom, severely bottlenecking its ultimate generalization capabilities. In stark contrast, GeoMin bypasses this sample-mining bottleneck through a unified geometric formulation, achieving both superior computational thrift and an uncompromised performance upper bound.

\section{Pseudo Code}
We describe the GeoMin pipeline in Algorithm \ref{alg:geomin}.

\begin{algorithm}[!h]
\caption{GeoMin}
\label{alg:geomin}
\KwIn{Datasets $\mathcal{D}_l, \mathcal{D}_u$, Policy $\pi_\theta$, Threshold steps $T_1$, Top-$K$, GMM threshold $\tau$}
\KwOut{Optimized policy $\pi_\theta$}

\BlankLine
\For{\text{step } $t = 1$ \KwTo $T$}{
    \tcp{\textbf{Stage 1: Supervised Anchoring ($t \le T_1$)}}
    \If{$t \le T_1$}{
        Sample labeled mini-batch $\mathcal{B}_l \sim \mathcal{D}_l$ and generate rollouts via $\pi_\theta$\;
        Fit online layer-wise vMF parameters $(\bm{\mu}_c^{(l)}, \kappa_c^{(l)})$ via Eq.~\eqref{eq:mean_vec} and \eqref{eq:mu_kappa}\;
        Identify ambiguous boundary sample set $\Omega \subseteq \mathcal{B}_l$ using Eq.~\eqref{eq:boundary_cond}\;
        Amplify advantages $\tilde{A}_k$ via Eq.~\eqref{eq:advantage_amplify} and update policy $\pi_\theta$\;
    }
    
    \BlankLine
    \tcp{\textbf{Stage 2: Semi-Supervised Sample Mining ($t > T_1$)}}
    \Else{
        Sample joint mini-batches $\mathcal{B}_l \sim \mathcal{D}_l$ and $\mathcal{B}_u \sim \mathcal{D}_u$\;
        Refresh vMF parameters and historical buffers $(\mathcal{H}_t^+, \mathcal{H}_t^-)$ using $\mathcal{B}_l$\;
        Select top-$K$ discriminative layers $\mathcal{I}$ maximizing separability $\Delta^{(l)}$ via Eq.~\eqref{eq:separability}\;
        
        \BlankLine
        \For{\text{each unlabeled } $q_i \in \mathcal{B}_u$}{
            Determine pseudo-label $a_i^*$ and compute confidence score $s_i$ over $\mathcal{I}$ via Eq.~\eqref{eq:conf_score}\;
        }
        Fit a 2-component GMM over $\{s_i\}$; Filter reliable pool $\mathcal{B}_u^* = \{ (q_i, a_i^*) \mid p(\mathcal{C}_{\text{high}} \mid s_i) > \tau \}$\;
        
        \BlankLine
        Optimize policy $\pi_\theta$ over the augmented training batch $\mathcal{B}_l \cup \mathcal{B}_u^*$\;
    }
}
\Return{Optimized policy $\pi_\theta$}\;
\end{algorithm}

\end{document}